%% file: PACIRM.tex
\documentclass[11pt,letterpaper]{article}
\usepackage[a4paper,bindingoffset=0.2in,%
left=0.8in,right=0.8in,top=1in,bottom=1in,%
footskip=.25in]{geometry}

\usepackage[utf8]{inputenc} 
\usepackage[T1]{fontenc}  
\usepackage{hyperref}    
\usepackage{url}      
\usepackage{booktabs}    
\usepackage{amsfonts}    
\usepackage{nicefrac}    
\usepackage{microtype}   
\usepackage{xcolor}     
\usepackage{wrapfig, authblk}

\usepackage{microtype, bbm}
\usepackage{graphicx}
\usepackage{booktabs} 
\usepackage{subfig}

\usepackage{amsmath}
\usepackage{amssymb}
\usepackage{mathtools}
\usepackage{amsthm}
\usepackage{pgf,tikz}
\usepackage[capitalize,noabbrev]{cleveref}

\theoremstyle{plain}
\newtheorem{theorem}{Theorem}[section]
\newtheorem{example}{Example}[section]

\newtheorem{lemma}[theorem]{Lemma}
\newtheorem{corollary}[theorem]{Corollary}
\theoremstyle{definition}

\newtheorem{assumption}[theorem]{Assumption}
\theoremstyle{remark}
\newtheorem{remark}[theorem]{Remark}
\newcommand{\indep}{\perp \!\!\! \perp}

\DeclareMathOperator{\E}{\mathbb{E}}
\let\Pr\relax\DeclareMathOperator{\Pr}{\mathbb{P}}

\DeclareMathOperator*{\argmin}{arg\,min}

\newcommand{\N}{\mathcal{N}}

\newcommand{\I}{\mathcal{I}}
\newcommand{\F}{\mathcal{F}}

\newcommand{\Pa}{\texttt{Pa}}

\newcommand{\D}{\mathcal{D}}
\newcommand{\R}{\mathcal{R}}

\newcommand{\At}{\texttt{At}}

\newcommand{\fc}{f_\circ}

\newcommand{\ThetaTr}{\Theta_{\text{train}}}

\usepackage[textsize=tiny]{todonotes}

\allowdisplaybreaks
\title{PAC Generalization via Invariant Representations}

\author[1]{Advait Parulekar\thanks{Email: advaitp@utexas.edu}}
\author[2]{Karthikeyan Shanmugam\thanks{This research was completed while the author was at IBM Research AI group, Yorktown Heights, NY 10598, USA. Email: karthikeyanshanmugam88@gmail.com}}
\author[1]{Sanjay Shakkottai\thanks{Email: sanjay.shakkottai@utexas.edu}}
\affil[1]{Department of Electrical and Computer Engineering, The University of Texas at Austin}
\affil[2]{Google Research India}

\begin{document}

\maketitle

\begin{abstract}
One method for obtaining generalizable solutions to machine learning tasks when presented with diverse training environments is to find \textit{invariant representations} of the data. These are representations of the covariates such that the best model on top of the representation is invariant across training environments. In the context of linear Structural Equation Models (SEMs), invariant representations might allow us to learn models with out-of-distribution guarantees, i.e., models that are robust to interventions in the SEM. To address the invariant representation problem in a {\em finite sample} setting, we consider the notion of $\epsilon$-approximate invariance. We study the following question: If a representation is approximately invariant with respect to a given number of training interventions, will it continue to be approximately invariant on a larger collection of unseen SEMs? This larger collection of SEMs is generated through a parameterized family of interventions. Inspired by PAC learning, we obtain finite-sample out-of-distribution generalization guarantees for approximate invariance that holds \textit{probabilistically} over a family of linear SEMs without faithfulness assumptions. Our results show bounds that do not scale in ambient dimension when intervention sites are restricted to lie in a constant size subset of in-degree bounded nodes. We also show how to extend our results to a linear indirect observation model that incorporates latent variables.
\end{abstract}

\section{Introduction}

A common failure of empirical risk minimization in machine learning is that it is beneficial to exploit spurious correlations to minimize training loss. A classic example of this comes from \cite{Beery2018RecognitionIT} where we have a dataset with pictures of cows and camels. In most cases, the cows are in pastures and the camels are in deserts, and a classifier that works well on such a dataset generalizes poorly to classifying cows in deserts. How do we prevent a classifier tasked with distinguishing pictures of cows from pictures of camels from using the color of the images? Certainly doing so might actually help with the classification - take an extreme example wherein rather than just having an indicative background, the species of the animal is annotated in the image itself. In this case, simply learning to interpret what is annotated will lead to an excellent classifier. However, such a classifier will perform poorly on a future dataset with incorrect or missing annotations.

In asking for such ``out-of-distribution" 
(OOD) generalization (generalization to different distributions rather than different samples from the same distribution), we are asking for more than what is guaranteed by traditional PAC learning. For instance, we might have a \textit{covariate shift}, meaning the marginal on the covariates of the joint distributions of our training and test data is different. We would like a classifier that is trained on training environments to generalize to the test environments. 

Recently, in the OOD generalization literature, the data for the training and test environments have been modeled to arise from a set of {\em causal models} that are intervened versions of each other. The true causal relationship in causal models between a target variable and its causal parents remains invariant while other relationships could change \cite{peters2016causal,bareinboim2012local}. Building on this observation, authors in \cite{arjovsky2019invariant} proposed to learn a representation across the training environments such that the optimal classifier on top of it remains the same. 
We refer to this as the \textit{invariant representation} in this work. Will this representation and the predictor generalize (remain invariant) in unseen environments? Does such a formulation result in improvements over other methods such as empirical risk minimization or distributionally robust optimization?

There has been some debate about this matter. Some theoretical works such as \cite{rosenfeld2020risks}, \cite{kamath2021does} describe regimes and settings in which either the convex penalty suggested by \cite{arjovsky2019invariant} or the invariance formulation in general are not sufficient to generalize to new environments. On the other hand, there are also settings described in \cite{arjovsky2019invariant}, \cite{ahuja2020empirical}, \cite{kamath2021does} in which IRM does provably lead to OOD generalization guarantees that are not possible using other methods. However, these prior studies require strong conditions such as faithfulness \cite{yang2018characterizing}, \cite{pearl_2009}, that come from the structure learning community. Indeed, as noted in \cite{uhler2013geometry}, faithfulness (that observed invariances imply structural constraints on the causal graph) is a very strong assumption in the finite sample setting. 
This is because the volume of the set of linear SEMs that are ``close'' to ones with faithfulness violations is a large constant fraction of all linear SEMs. This makes it difficult to resolve the question of whether an observed conditional independence truly holds due to the underlying causal structure or not with finite samples. Finally, invariance based results almost always study the infinite sample setting and also assume general position conditions on the training environments; similar to the faithfulness assumption, these become much stronger assumptions in the finite sample setting. 

Thus in the worst case, perhaps we cannot expect generalization guarantees without an exponentially many interventions/samples. Instead of studying the worst case setting, suppose we consider the \textit{PAC} setting -- can we now derive finite sample generalization guarantees without the need of strong faithfulness or general position assumptions, yet with a polynomial scaling in interventions/samples?

\subsection{Contributions}\label{section:contributions}
In this paper, we provide the {\em first known generalization and associated finite-sample guarantees for invariant representations without faithfulness assumptions.}
We work in a setting in which we have access to a family of linear SEMs related to each other by hard and soft interventions on arbitrary subsets of variables. We assume that the training environments arise by random sampling from a distribution over this family. We derive a PAC bound for the number of interventions needed to get probabilistic generalization over the family of SEMs. A central result is that we need only $O(\frac{n^4}{\delta'})$ training environments, such that approximately invariant representations on the training set generalizes with probability $1-\delta'$ on the family of unseen linear SEMs.
%
We show tighter interventional complexity independent of $n$ for SEM families that have simultaneous atomic interventions on any fixed set of $k$ nodes and soft interventions on $k$ nodes in a degree bounded setting.

Secondly, we characterize the number of samples in each training environment that is needed for approximate invariance to hold with high probability. This gives an end to end sample complexity guarantee along with the interventional complexity required in the above setting without any faithfulness or general position assumptions. Furthermore, we show extensions of our results to a linear indirect observation model that incorporates latent variables.

Finally, we empirically demonstrate that in the setting described above using a notion of generalization that we describe, \textit{most} approximately invariant representations generalize to \textit{most} new distributions.

\subsection{Related Work}
\textbf{Causal Structure Learning:} One approach, in the presence of data sampled from structural models, is simply to try and learn the DAG structure explicitly. Algorithms that use conditional independence (CI) tests and/or penalized likelihood scoring are used to determine the DAG \cite{Spirtes2000,solus2021consistency,chickering2020statistically,brouillard2020differentiable}. These methods 
can only learn DAGs up to what is called the Markov Equivalence Class (MEC) \cite{Verma1990}, or $\I$-MECs (Interventional MECs) when given access to general interventions \cite{hauser2012characterization,squires2020permutation,brouillard2020differentiable} \cite{yang2018characterizing} under some form of faithfulness assumptions on the underlying causal model. It has been shown that some forms of faithfulness assumption are actually stronger than one might expect; linear SEMs that are close to a faithful violation form a much larger set \cite{uhler2013geometry} and are difficult to distinguish in finite samples from faithful ones. However, in light of the fact that we only want generalization guarantees for invariance and not structure recovery, one could argue that learning the structure, or even feature selection 
is stronger than what we actually need.


\textbf{Domain Adaptation Methods:} Domain adaptation literature \cite{ganin2016domain, muandet2013domain,ben2007analysis} learns representations who distributions does not change across training and an unlabeled test distribution. Another line of work searches for the best mixture of training distributions to train on or mixture of pretrained models for generalization to unlabeled test data or test data with limited test labels \cite{mansour2021theory,mansour2009domain}. Distributionally robust optimization approaches that optimize the worst risk over an uncertainty set of distributions have been proposed \cite{duchi2021statistics,sagawa2019distributionally}. Please see Appendix A of \cite{gulrajani2020search} for a more exhaustive list of works. Robust optimization and multi source domain adaptation methods search in an uncertainty ball around the training distributions while domain adaptation fail even with a shift in marginal distributions of the labels \cite{zhao2019learning}. Such methods fail when the test distribution is outside the convex hull of training distributions.

\textbf{Finite Samples and Invariant Risk Minimization:} One framework proposed to bypass structure learning difficulties and directly look for invariant representations is to use \textit{Invariant Risk Minimization} (IRM) \cite{arjovsky2019invariant}. \cite{arjovsky2019invariant} show that under some general position assumptions on the population covariances of the training distributions the exact invariant representation recovers the true causal parents. The work of \cite{ahuja2020empirical} makes similar assumptions to derive sample complexity guarantees. \cite{kamath2021does} also consider the question of getting generalization guarantees from IRM for general distributions,
and address the number of training intervention needed for IRM; however, they work in the infinite sample limit with exact invariance and require an analytical mapping between intervention index set and training distributions. All these works assume regularity conditions and provide deterministic guarantees for a test distribution.

\textbf{Invariant Predictors on DAGs:} 
Recent works \cite{subbaswamy2019preventing,magliacane2017domain} have assumed the knowledge of the causal DAG behind the unseen test and/or train distributions along with information about the nodes that are intervened on in the test. It is then possible to characterize all possible invariant representations as a function of the graph. In our study, we do not assume any such side information about the SEM underlying our distributions.

\section{Model}\label{section:model}
\subsection{Structural Equation Model}
We consider a causal generative model specified by a linear Structural Equation Model (SEM) - data that is generated by linear equations on a directed acyclic graph (DAG) of variables. Specifically, we work with a random $n+1-$dimensional vector $X = (X_1, X_2, \cdots, X_n, X_t)$. We assume that there is some unknown directed acyclic graph $G$ with nodes $(v_1, v_2, \cdots, v_n, v_t)$ which specifies the relationships between the variables in $X$, so that the node $v_i$ corresponds to the random variable $X_i$. The variables are related to each other by structural equations of the form:
\begin{align}
X_i &= \left(\beta_i^{\theta}\right)^\top \Pa_i+\eta_i = \sum_{j\in \Pa_i}\beta_{j, i}^{\theta}X_j + \eta_i^{\theta},\label{eq:SEM}\\
X_t &= \beta_t^\top \Pa_t + \eta_t = \sum_{j\in \Pa} \beta_{j, t}X_j+\eta_t.\label{eq:targetSEM}
\end{align}
Here $\eta_i$ models the exogenous variables in our system and $X_i$ denote the endogenous variables in the system. $\Pa_i$ denotes the set of parents of node $v_i$ in $G$. The weights of the linear SEM, when there is a directed edge between node $v_j$ and $v_i$ is represented by $\beta_{j, i}$. There is some target node $X_t$ that is considered to be special in that we would like to predict its value given the rest of the variables. We will denote the vector consisting of the rest of the variables as $X_{-t}$. We will denote by barred quantities the zero-padded versions of the above vectors, so $\overline{\beta}_t$, is such that $\beta_t^\top \Pa_t = \overline \beta_t^\top X_{-t}$.

\subsection{Interventions}
An important feature of Equation (\ref{eq:SEM}) is the dependence of $\beta^\theta_i$ on $\theta\in \Theta$, an element of some \textit{intervention} index set $\Theta$. $\Theta$ specifies a paramterized intervention family, i.e. a family of linear SEMs related to each other by interventions. Conceptually, in a single interventional environment, the mechanism by which causal parents of a variable influence it is changed. The changes in the mechanisms are reflected in the weights $\beta_{j,i}^{\theta}$ parameterized by $\theta$. Note that the conditional distribution of the target variable $X_t$ given its parents $\Pa_t$ is assumed to be the same even for different $\theta\in \Theta$. This is a property that we would like to exploit so that our task of predicting $X_t$ is robust to different interventions in $\Theta$, i.e. there exists at least one robust predictor for all environments. We begin with an observational distribution with weights given by $\beta^{\theta_\circ}$. There are two common types of interventions studied in the literature.

\textit{Atomic or Hard Interventions:} In an \textit{atomic} or \textit{hard} intervention $\theta$ at some node $v_i$, we assign a value to $X_i$. This corresponds to setting $\beta_{j, i} = 0$ for all $j\in \Pa_i$, and setting $\eta_i^{\theta} = a_i^\theta$.

\textit{Soft Interventions: }In a \textit{soft} intervention $\theta$, we modify the weights $\beta_i^\theta\ne \beta_i^{\theta_\circ}$ while keeping the noise the same $\eta_i^{\theta} = \eta_i^{\theta_\circ}$.

In this sense, the data is generated from some joint distribution $\Delta_\theta$, and for training purposes we see samples of the form $\{(\theta, x_{-t}, x_t)\}_{\theta\in \ThetaTr}$ for some $\ThetaTr\subset \Theta$. That is, we see data from a number of training data-sets indexed by intervention. 

\textbf{Distributions over Interventions and Sampling Model:} In our model, we consider a distribution $\D$ over $\Theta$ from which we assume that our training and test data-sets are drawn by sampling the intervention index $\theta$ independently and randomly from $\D$. We provide generalization bounds that work with high probability over the randomness of the choice of test distribution, rather than deterministically as has been done in prior works (as we will explore in more detail, providing deterministic bounds is not possible without additional assumptions). Formally, $\ThetaTr$ is a set of intervention index set samples drawn i.i.d from $\D$ over $\Theta$. The test intervention index $\theta_{\text{test}}$ is also drawn from $\D$ independently.
\begin{remark}
One setting in which random interventional distributions occur is in conditional sub-sampling (see for example, \cite{peters2016causal}, \cite{shah21Treatment}, \cite{shah22aFinding}). Here we are only provided with an observational dataset $(\textbf{x}_i)_{i\in [N]}$. We artificially generate interventions using some rule $f(\textbf{x}_S)\in [m]$, where $\textbf{x}_S$ denotes some subset of $\textbf{x}$ indexed by a subset $S$ of all nodes. Here the interventional datasets are taken to be $(N_j)_{j\in [m]}$ where $N_j = \{\textbf{x}_i: f((\textbf{x}_i)_S) = j\}.$
\end{remark}

\section{PAC-Invariant Representations}\label{section:PACIRM}
Motivated by Invariant Risk Minimization (IRM), we consider the search for a model that performs well under a variety of distributions in the hopes of attaining generalization guarantees. Consider a prediction function $f$ with loss $\R^\theta(f)$. Here the superscript $\theta$ represents the intervention index. For example, in least squares regression, $f\in \mathbb{R}^n$ is a vector, and the loss is given by $\R^\theta(f) = \E_{\theta}[\left (f^\top X_{-t}-X_t\right)^2]$ where the expectation is taken with respect to $\Delta_\theta$. We are interested in generalization guarantees for representations $\Phi$ that satisfy the property that 
\begin{equation}\label{eq:IRM}
f = \argmin_f \R^\theta(f \circ \Phi)\hspace{0.5cm} \forall\theta\in \ThetaTr,\end{equation} is invariant across environments. That is, the least squares solution on top of the representation is the same for every intervention. We will denote the set of all invariant representations over a class of SEMs $\Theta'$ by $\I(\Theta')$. We refer to the full model $f\circ\Phi$ when $\Phi$ is invariant as an \textit{invariant solution}. We use $\R^\theta_\Phi(f)$ to denote the loss for $f$ on top of representation $\Phi$, so $\R^\theta_\Phi = \R^\theta(f\circ\Phi)$. We refer to $f$ as the \textit{head} of the model.

Feature selection, a motivation for the study of invariant representations, corresponds to diagonal representations; henceforth, we focus on the class of representations {\color{black} that are feature selectors}, given by diagonal matrices.



\textbf{Invariance via gradients: }

Rather than work with the loss itself, we follow \cite{ahuja2020empirical}, \cite{arjovsky2019invariant} and use the \textit{gradient} of the loss instead (see Lemma \ref{lemma:lossisgrad}).
$$
\Phi\in \I(\ThetaTr)\iff\exists f \text{ s.t. }\nabla_f \R^\theta_\Phi(f) = 0\hspace{0.2cm}\forall \theta\in \ThetaTr.
$$


Unlike previous work, we do not further minimize the weighted empirical loss over training environments. We will instead only consider the generalizability of \textit{any} approximately invariant representation.

\textbf{Finite samples and $\epsilon-$approximately Invariant Representations: }Since we are working with finite samples,
we can no longer hope for \textit{exact} invariance across environments. We slightly change the definition of invariance to be about the gradient being close to $0$.
\begin{equation}\label{eq:approximateGradientIRM}
\Phi\in \I^\epsilon(\ThetaTr)\iff\exists f \text{ s.t. }\Vert \nabla \R^\theta_\Phi(f)\Vert_2\le \epsilon ~\forall \theta\in \ThetaTr.
\end{equation}

We refer to approximate versions of quantities using a superscript $\cdot^\epsilon$, and, given specific datasets, we refer to finite sample versions of these quantities using hats, so for instance the set of representations that are invariant for some set $\Theta'$ is denoted $\I(\Theta')$, the set of $\epsilon-$approximately invariant representations is given by $\I^\epsilon(\Theta')$, finally, given a particular dataset, the set of $\epsilon-$approximately invariant representations would be denoted $\hat \I^{\epsilon}(\Theta')$. Note that the finite sample quantities are random (from the randomness of the samples drawn from $\Delta_\theta$), while the $\epsilon-$approximate quantities are deterministic. The dataset of $N$ points is denoted in matrix form as $\textbf{X}$ with $\textbf{X}_{-t}$ and $\textbf{X}_t$ denoting the non-target matrix and target vector respectively. 

\textbf{Significance of linear representations:} A linear representation $\Phi$ affects invariance non-trivially by selecting an ``effective" column space for the regression, that is, in determining the column space of $\Phi^\top \textbf{X}_{-t}$. Because we are looking for a fixed (across environments) least-squares head on top of the representation, $\Phi$ can be further composed with any invertible linear map, only to be inverted at the head to obtain another invariant solution $(\Phi A, A^{-1}f)$ from an invariant solution $(\Phi, f)$. Because of this freedom, we can actually choose to work with some fixed head ahead of time, say $(1, 0, \cdots, 0)$. This is formalized in Lemma \ref{lemma:headdoesntmatter}. As we will see, a key construction is that of $\I^\epsilon_{\fc}(\theta)$, defined as those $\Phi$ such that the gradient is near-zero at $\fc$, that is, $$\Phi \in \I^\epsilon_{\fc}(\theta) \iff \Vert \nabla \R^\theta_\Phi(\fc)\Vert \le \epsilon,$$ for some specific $\fc$. Similarly, for some set of interventions $\Theta'$, we define $\I^\epsilon_{\fc}(\Theta') = \bigcap_{\theta\in \Theta'}\I^\epsilon_{\fc}(\theta)$. This allows us to look at invariance as a property of a representation and a \textit{single} intervention, rather than that over collection of interventions as in Equation \ref{eq:IRM}. Why this is important is discussed further in Remark \ref{remark:fixing_head}.

\subsection{Motivation for Probabilistic Guarantees}\label{section:probabilistic_guarantee}
Consider a distribution $\D$ over interventions. During training, we see samples from specific interventions drawn from this distribution $\ThetaTr$, and compute invariant representations for these interventions. Ideally, we would like to be able to say that an invariant solution computed in this way will generalize to all future interventions on the same DAG. However, this is not possible without further assumptions, as the following examples illustrate.

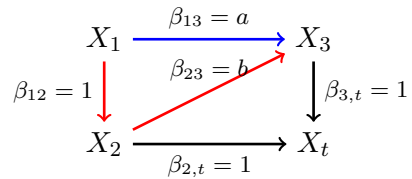
\begin{wrapfigure}{r}{0.4\textwidth}
\centering
\begin{tikzpicture}
\node[text centered] (x1) {$X_1$};
\node[below = 0.8 of x1, text centered] (x2) {$X_2$};
\node[right = 2 of x1, text centered] (x3) {$X_3$};
\node[below = 0.8 of x3, text centered] (y) {$X_t$};

\draw[->, red, line width= 1] (x1) -- node[left,black,font=\footnotesize]{$\beta_{12} = 1$} (x2);
\draw [->, blue, line width= 1] (x1) -- node[above,black,font=\footnotesize]{$\beta_{13} = a$} (x3);
\draw [->, red, line width= 1] (x2) -- node[above,black, font=\footnotesize]{$\beta_{23} = b$} (x3);
\draw [->, line width= 1] (x2) -- node[below,font=\footnotesize]{$\beta_{2, t} = 1$} (y);
\draw [->, line width= 1] (x3) -- node[right,font=\footnotesize]{$\beta_{3, t} = 1$} (y);
\end{tikzpicture}
\caption{Figure for Example \ref{example:faithfulness} showing that there could be more independencies in the data than apparent from the DAG alone.}\label{fig:faithfulness}
\end{wrapfigure}

\begin{example}[Faithfulness]\label{example:faithfulness}

Consider a linear SEM on four variables $X_1, X_2, X_3, X_t$ given by 
\begin{align*}
  X_1 &= \N(0, 1),&& X_2 = X_1 + \N(0, 1),\\
  X_3 &= aX_1 + bX_2 + \N(0, 1),&&X_t = X_2 + X_3 + \N(0, 1).
\end{align*}
depicted in Figure \ref{fig:faithfulness}. This system has at least the following two invariant representations of $X = (X_1, X_2, X_3)$ if $a = -b$ for every intervention: $(X_2, X_3)$ and $(X_2)$. However, only the first continues to be invariant when $a\ne -b$. In other words, given an arbitrary number of training environments, each of which satisfies $a = b$, we might decide that $(X_2)$ is an invariant classifier. However, this fails to be invariant once we include a test intervention for which $a\ne -b$. This happens because the joint distribution is not \textit{faithful} to the DAG provided. In short, this means that the conditional independencies indicated by the DAG are not the only ones found in the data. In our case, the effect of $X_1$ on $X_3$ through the blue path and the red path exactly cancel in all training interventions.
\end{example}
\begin{example}[Degenerate Interventions]\label{example:smallPerturbation}
Consider the degenerate situation in which each of the interventions is actually the same. The Empirical Risk Minimization (ERM) solution itself is invariant, however, the ERM solution comes with no generalization guarantees to other interventions. It is clear that some ``diversity" condition among the training environments is necessary to get generalization guarantees. 
\end{example}
The key idea is that in such situations one might say that these issues (faithfulness violations as in Example \ref{example:faithfulness} or degeneracies as in \ref{example:smallPerturbation}) will likely continue to manifest in future distributions. In particular, if we assume that both training and test intervention come from a single distribution over interventions, can we at least say that with high probability over a test intervention also drawn from $\D$, we will generalize to that intervention? 
%
More formally, we ask the following:

\textit{How big should $\ThetaTr$ be so that we can say that for $\theta\sim\D$, $\I_{\fc}^\epsilon(\ThetaTr) \in \I_{\fc}^\epsilon(\theta)$ with high probability?}

In other words, we would like invariance on our training interventions to \textit{certify} invariance over a future interventions with high probability. To further clarify the nature of the guarantee we are considering, we look at the following simple example.

\begin{wrapfigure}{r}{0.5\textwidth}
\centering
\begin{tikzpicture}
\node[text centered] (x1) {$X_1$};
\node[right = 1 of x1, text centered] (x2) {$X_2$};
\node[right = 1 of x2, text centered] (x3) {$X_{n-1}$};
\node[right = 1 of x3, text centered] (x4) {$X_{n}$};
\node[below = 0.5 of x4, text centered] (y) {$X_t$};

\draw[->, line width= 1] (x1) -- (x2);
\draw[->, dashed, line width= 1] (x2) -- (x3);
\draw[->, line width= 1] (x3) -- (x4);
\draw[->, line width= 1] (x1) -- (y);
\draw[->, line width= 1] (x4) -- (y);

\end{tikzpicture}
\caption{DAG for intervention in Example \ref{example:triangleDAG}}\label{fig:probabilistic_guarantee}
\end{wrapfigure}
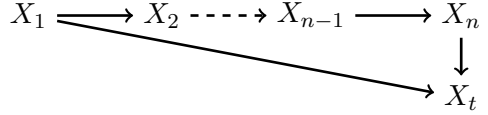
\begin{example}
\label{example:triangleDAG}
Consider the DAG given in Figure (\ref{fig:probabilistic_guarantee}). Consider the following distribution over interventions: We set $X_1=1$. Independently and with probability $1/2$ for each node that is not $X_1$ or $X_t$, we assign its value to be $0$. Else we set its value equal to its parent. The only invariant representations in this case are $(X_n)$ and $(X_1, X_n)$. Now consider any training set of $m$ interventions. If $m \ll 2^n$, we expect to never see the intervention in which \textit{none} of the edges are disconnected (i.e., an environment with no intervention on any of $X_2, X_3, ..,, X_{n-1}$), in which case $(X_1)$ also appears to be invariant over every observed intervention. However, it fails to be invariant on the disconnected intervention described above,
which occurs with an extremely low probability of $1/2^{n-1}$.



Our result instead gives a probabilistic guarantee that says that with a certain probability over the randomness from which we are drawing interventions, any representation that appears to be invariant over the training set will also be invariant when the set is expanded to include another i.i.d. intervention (the \textit{test} intervention). As demonstrated in this example, we may need $\Theta(2^n)$ interventions to get deterministic guarantees while our $\text{poly}(n)$ result highlights the benefit of getting a probabilistic guarantee in this setting.
\end{example}

\subsection{PAC Formulation}
In this section, we show that the question of generalization of $\epsilon-$invariant representations introduced in the previous section can be rewritten as a question of PAC generalization. 
\subsubsection{VC dimension and PAC generalization}
Let $X\subset \mathbb{R}^d$ be a subset of points. A function class $\F$ is said to \textit{shatter} $X$ if every binary assignment to $X$ is realized by some element of $\F$, that is, for any $\sigma:X\to \{\pm 1\}$, we have that there is some $f\in \F$ such that $f(x) = \sigma(x)$ for all $x\in X$. The VC dimension of a function class is defined as the size of the largest set that is shattered by it.
The following classical result from PAC learning theory allows us to determine how many interventions are needed to generalize.

\begin{lemma}[\cite{shalev-shwartz_ben-david_2014}]\label{lemma:PAC}
Consider a class $\F$ of functions from $\mathcal{X}$ to $\{\pm 1\}$ of VC dimension $d$, and a distribution $\tilde{\D}$ over $\mathcal{X}\times\{\pm1\}$. In the realizable setting, that is, when there exists $f\in \F$ such that $\Pr_{(X, Y)\sim \tilde{\D}}[f(X)\ne Y] = 0$, given $m = O(\frac{d+\log \frac{1}{\delta_1}}{\delta_2})$ samples $(x_i, y_i)_{i = 1}^m$ and $\tilde f$ such that $\tilde f(x_i) = y_i$ for all $i$, with probability at least $1-\delta_1$, $\tilde f$ satisfies $\Pr_{(X, Y)\sim \tilde{\D}}[\tilde f(X)\ne Y] < \delta_2$.
\end{lemma}
We will need the following known VC dimension.
\begin{lemma}[\cite{shalev-shwartz_ben-david_2014}]\label{lemma:halfspaceVC}
The VC dimension of halfspaces in $\mathbb{R}^d$, that is, the function class 
$\F = \big\{\mathbbm{1}\{a^\top x < b\}: a\in \mathbb{R}^d, b\in \mathbb{R}\big\}$ is $d+1$. 
\end{lemma}
\subsubsection{PAC Invariance}
We will use PAC learning theory to study generalization for representations. For that we need to rephrase the invariance problem to being one of the generalization of binary classifiers. The proof is deferred to the Appendix, Lemma \ref{appendix_lemma:invariance_to_classification}.

\begin{lemma}
There is a function class $\F$ mapping interventions $\theta\in \Theta$ to $\{0, 1\}$, such that
$$\Phi\in \I_{\fc}^\epsilon(\ThetaTr) \iff \exists f_\Phi^\epsilon \in \mathcal{F} \text{ such that } f^\epsilon_\Phi(\theta) = 1 ~\forall \theta\in \ThetaTr.$$ 
\end{lemma}

In summary, if we can bound the VC dimension of $\F$, we can specify interventional complexity guarantees. The proof is deferred to the Appendix, Corollary \ref{appendix_cor:VCtoPAC}
\begin{corollary}\label{cor:VCtoPAC}
If the VC dimension of $\F$ is bounded by $d$, then given at least $O(\frac{d+\log\frac{1}{\delta}}{\delta'})$ training interventions, if $\I^\epsilon(\Theta)\ne \{\}$, we have with probability $1-\delta$ over the set of training interventions, 
$$\Pr_{\theta \sim \D} (\I^\epsilon_{\fc}(\ThetaTr)\not\subset \I^\epsilon_{\fc}(\theta)) \leq \delta'. $$
\end{corollary}
Finally, there is at least one truly invariant representation, as demonstrated in the following Lemma.

\begin{lemma}\label{lemma:realizable}
There exists at least one invariant representation.
\end{lemma}
\begin{proof}
The representation $\Phi = \text{diag}(\overline \beta_t)$ is invariant. That is, the diagonal matrix with $\beta_t$ on the diagonal in indices corresponding to $\Pa_t$.
\end{proof}

\begin{remark}\label{remark:fixing_head}[Fixing a head is important for PAC learning] 
Considering a fixed head $\fc$ allows us to certify invariance locally, that is, looking at only a single intervention at a time. This is in contrast with the previous characterization, Equation (\ref{eq:IRM}), in which we simply ask for the best head on top of a representation to be the same for all interventions - a condition that we can only check for by considering all interventions simultaneously. 
\end{remark}

\subsection{The VC dimension of certain classes of interventional distributions}

 In this section we derive upper bounds on the number of interventions required to certify approximate invariance during training for specific instances of intervention parameterizations with infinite samples. For arbitrary interventions, we show that $O(n^4)$ interventions suffices. This result can be extended to other observation models; please see Appendix \ref{section:intervention_models_appendix} for additional discussion. Importantly, we extend to an indirect observation model in which we observe a linear transformation of the underlying linear SEM and notably, this model allows presence of latent variables. This is described in Theorem \ref{appendix_thm:indirect}. 
Further, we show that for atomic interventions on some fixed set of $k$ nodes, $O(k^4)$ training interventions suffices. For soft interventions on $k$ nodes, assuming that each node has in-degree bounded by $d$, we show that $O(d^{4k})$ interventions suffices. In the next section, we will show how to extend these results to the finite sample setting.

\subsubsection{General Interventions}
For this section we approach the task of bounding the VC dimension from the perspective of the complexity of the representation space. 
The class of interventions we consider are any interventions that lead to joint distributions over the variables that leave fixed the conditional distribution for the target variable given its parents.
%

\begin{theorem}\label{thm:general}

Suppose that we are given $O(\frac{n^{4}+\log\frac{1}{\delta}}{\delta'})$ interventions drawn independently from distribution $\D$ over the intervention index set $\Theta$. Then with probability at least $1-\delta$ over the randomness in $\ThetaTr$, the following statement holds: 
$$\Pr_{\theta \sim \D} (\I^\epsilon_{\fc}(\ThetaTr)\notin \I^\epsilon_{\fc}(\theta)) \leq \delta'. $$


In other words, an $\epsilon$- approximate invariant representation on the training environments generalizes with high probability to the test environment.
\end{theorem}


\textbf{Connections to Causal Bayesian Networks:} Invariance Principle or (less popularly known as) the modularity condition can be equivalently used to define causal Bayesian Networks \cite{bareinboim2012local} - the central object in Pearlian Causal Models. With respect to all possible interventions wherein variable $y$ has not been intervened on, only the representation that involves the true causal parent of $y$ is invariant. In other words, the property of invariance can be taken to be the signature of causality. An open question in the setting has been the following: In how many interventional environments does this invariance property need to hold before it generalizes to most unseen environments. With respect to a random sampling on interventional environments for linear SEMs, our result answers this via high probability generalization guarantees without any faithfulness assumptions. 

\textbf{Structure Learning and Faithfulness:} 
Invariance testing between a given set of interventional environments has been used to constrain the space of causal models \cite{yang2018characterizing} in the literature. However, the learning algorithm that synthesizes various invariances does require some notion of interventional faithfulness, i.e. observed invariances imply topological constraints on the true causal graph. 
In contrast, we do not make any such assumptions about faithfulness.

\textbf{Comparison to Regularity Conditions in Invariant Prediction:} Recently, invariant prediction
\cite{arjovsky2019invariant} has been used for out of distribution generalization. However, exact invariance holding deterministically in \textit{all} unseen distributions from a family of linear SEMs was desired. This required some general position conditions on the population covariances of the training interventional environments. The number of environments despite these additional conditions, required was $O(n^2)$. Our result on the interventional complexity is weaker but holds when we need no such assumptions on covariances and it holds under random sampling for approximate invariances. Similar technical general position conditions were also required in \cite{ahuja2020empirical} for generalization.

In another recent work \cite{kamath2021does}, if the mapping between the intervention index set $\Theta$ and the observed training distributions on $X$ is analytic, then only two training environments (almost surely) suffice for exact invariance in the population setting (i.e., infinite samples) for the entire index set. Here, we analyze approximate invariance without any restrictions on the mapping being analytic. For instance, their result does not apply in Example \ref{example:triangleDAG}, where we need two \text{specific} interventions (one in which each node is assigned the value of the parent, and any other) to certify invariance. In fact, any set of interventions needs to have the specific intervention highlighted in the example to certify invariance, and this happens with very low probability. 


\section{Finite samples}\label{sec:finite_samples}
Note that the above discussion was about population statistics. In reality, we only see finitely many samples from each interventional distribution. We use an estimator similar to that of \cite{ahuja2020empirical} to get finite sample guarantees. We need the following assumption for scale.
\begin{assumption}
The following bounds hold, $\Vert \Phi\Vert_2 \le 1$, $\Vert X\Vert_\infty < 1$.
\end{assumption}
\begin{lemma}\label{lemma:singlePhiMain}
Given $\frac{4nL^2}{\epsilon^2}\left(\log\frac{2n}{\delta}+n^2\log(1+\frac{8n^{3/2}}{\epsilon})\right)$ samples from $\Delta_\theta$, we have that with probability $1-\delta$ over the samples drawn \textit{in each interventional distribution}\footnote{In contrast to claims over the randomness $\D$ over $\Theta$ from previous theorems, this is a statement about the randomness $\Delta_\theta$ for $X$.}
$$\hat \I^\epsilon_{\fc}(\theta) \subseteq \I^{2\epsilon}_{\fc}(\theta).$$
\end{lemma}
\begin{proof}[Proof Sketch]
We show this using the standard concentration arguments for a single fixed representation $\Phi$. We then take a union bound over an $\epsilon-$net of the $\R^\theta_\Phi$ to get a uniform bound over all representations. See Lemma \ref{lemma:overallfinitesample} for a full proof.
\end{proof}
We can now take a union bound over all $\theta\in \ThetaTr$. In conclusion, we have shown that with high probability over the randomness in the samples we see in each of our interventional distributions, $\epsilon-$approximate invariance over the training data certifies $\epsilon-$approximate invariance over a $1-\delta'$ fraction of out full intervention set. The proof is deferend to the Appendix Theorem \ref{appendix_thm:main_thm}
\begin{theorem}
Given 
$$m = 
\begin{cases}O(\frac{k^4+\log\frac{1}{\delta}}{\delta'}) & \text{$\Theta$ is $k$ nodes, hard interventions}\\
O(\frac{d^{4k}+\log\frac{1}{\delta}}{\delta'}) & \text{$\Theta$ is $k$ nodes, soft interventions}\\
O(\frac{n^{4}+\log\frac{1}{\delta}}{\delta'}) & \text{$\Theta$ any interventions}\\
\end{cases}$$
interventional datasets, and $\frac{4nL^2}{\epsilon^2}\left(\log\frac{2nm}{\delta}+n^2\log(1+\frac{8n^{3/2}}{\epsilon})\right)$ samples in each dataset, we have that with probability $1-\delta$, with probability $1-\delta'$ for $\theta\sim\D$,
$$\hat\I^\epsilon_{\fc}(\ThetaTr) \subseteq \I^{2\epsilon}_{\fc}(\theta).$$
\end{theorem}
\section{Empirical Study}\label{section:conclusion}
In this section we highlight the above results empirically. 
The experiment is described in detail in Appendix~\ref{appendix:empirical}.
\begin{figure}
  \centering
  \begin{minipage}{0.3\textwidth}
    \centering
    \includegraphics[width=0.9\textwidth]{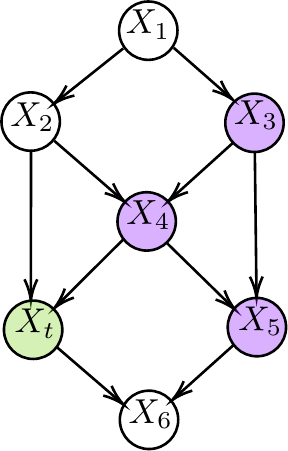} 
    \caption{Linear SEM for Section \ref{section:conclusion}. Nodes $\{v_3, v_4, v_5\}$ (purple) are used as intervention sites for $\D_{\text{hard}}$ and $\D_{\text{soft}}$. Each edge represents a weight of $1$. Node $X_t$ (green) is taken to be the target node.}\label{fig:toy_example}
  \end{minipage}\hfill
  \begin{minipage}{0.65\textwidth}
    \centering
    \includegraphics[width=0.9\textwidth]{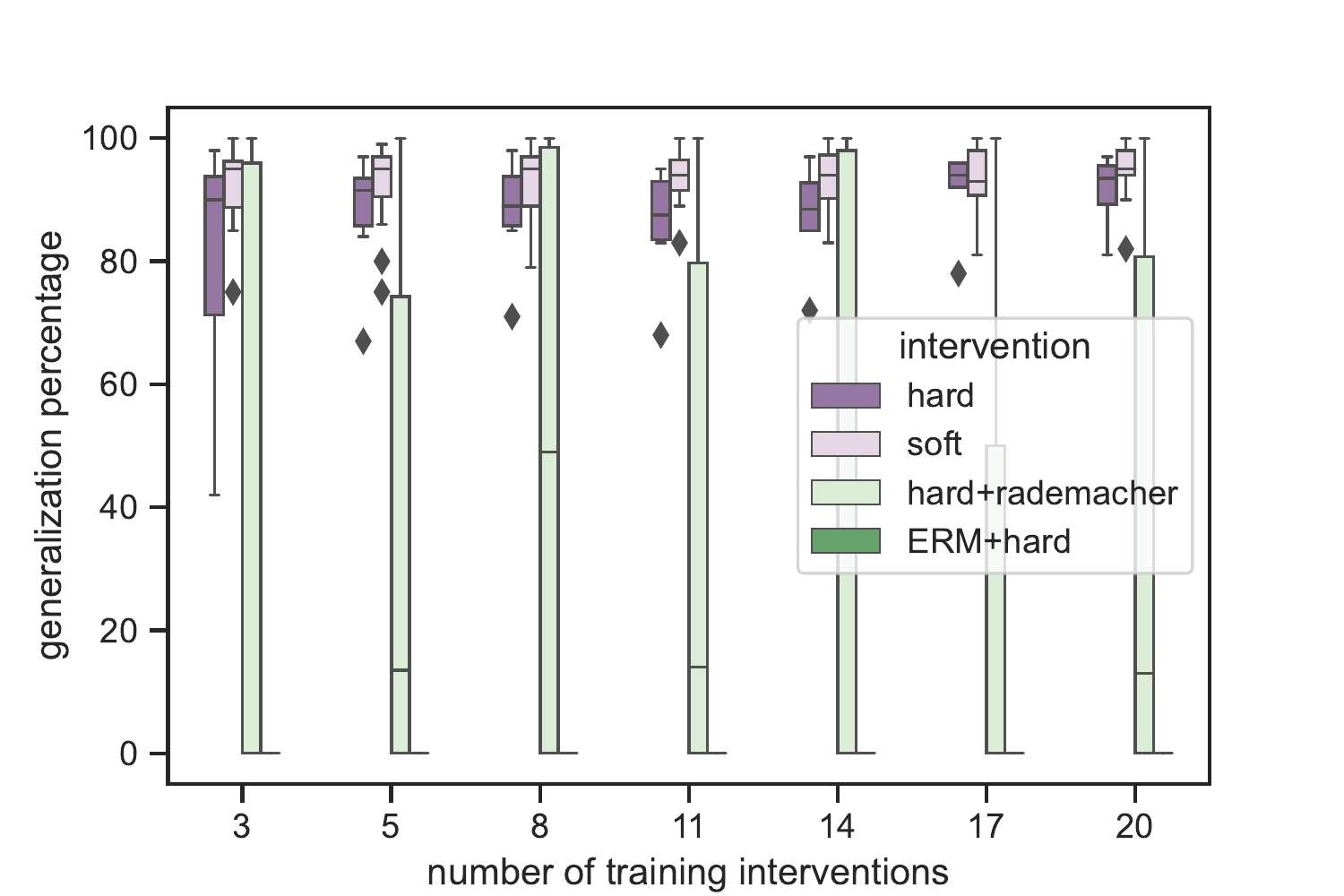} 
    \caption{Generalization of approximately invariant representations under different combinations of train and test interventional distributions for the experiment described in Section \ref{section:conclusion}. The ERM solution almost always fails to generalize, and the corresponding box plot is at `0'. }\label{fig:generalization}
  \end{minipage}
\end{figure}
%
%
%
We consider the $7$-node linear SEM in Figure \ref{fig:toy_example}. The target variable is taken to be $X_t$. Each edge weight is set to $1$ for the observational distribution. We consider an interventional distribution $\D_{\text{hard}}$ with support over the set of hard interventions on nodes $\{v_3, v_4, v_5\}$. Recall that a hard intervention consists of assigning a value to a node. We draw $m$ interventional distributions from $\D_{\text{hard}}$ as our training interventions, and draw a sample consisting of $N = 200,000$ datapoints from each distribution. We use a notion of inter-dataset variance of the least squares solutions to construct a set $s_{\text{hard}} = \{S_1, S_2, \cdots,\}$ where each $S_i$ is an approximately invariant representation derived from the training datasets. We then generate an additional $m^{\text{test}}=100$ test interventions from $\D_{\text{hard}}$, and generate datasets of size $N=200,000$ from each of them. We use the same inter-dataset variance to count the percentage of subsets in $s_{\text{hard}}$ that continue to have low variation between the least squares solutions on the new {\em test distributions} and the {\em average least squares solution on the training distributions}. We repeat the above for soft interventions drawn from $\D_{\text{soft}}$. We consider soft interventions that modify the weights of the edges into $v_3, v_4,$ and $v_5$. Please see Appendix \ref{appendix:empirical} for exact details about how the hard and soft interventional ensembles are defined.

We also consider a different interventional distribution $\D_{\text{rad}}$ to test the generalization of representations that are approximately invariant given training distributions drawn from $\D_{\text{hard}}$. These interventions are generated by randomly flipping the signs of the edge weights in the original SEM.

Finally, to confirm that there is indeed variation across the least squares solutions in the different environments, we also plot the percentage of test distributions to which the ERM solution on the training distributions is similar to ERM solution of the test distributions (this percentage is almost always 0; meaning that the ERM solution always fails to generalize). We repeat this for $m$ varying beteen $3$ and $20$. The results are plotted in Figure \ref{fig:generalization}.

\textbf{Interpretation:} While not every subset in $s_{\text{hard}}$ and $s_{\text{soft}}$ generalizes to every test dataset, we see that \textit{most} subsets generalize to \textit{a large percentage} of test distributions. Furthermore, the percentage that generalize when the train and test interventions are identically distributed exceeds the percentage that generalizes to datasets that come from the rademacher interventional family. This phenomenon is captured in our PAC bounds, which describe when \textit{most} approximately invariant representations will continue to be approximately invariant on \textit{most} test distributions. 

\section*{Acknowledgements}

This research  is supported in part by NSF Grants 2019844 and 2112471, ONR Grant N00014-19-1-2566, the Machine Learning Lab (MLL) at UT Austin, and the Wireless Networking and Communications Group (WNCG) Industrial Affiliates Program.

\bibliographystyle{alpha} 
\bibliography{PACIRM}

\newpage
\appendix

\input{PACIRM_appendix}

\end{document}

%% file: PACIRM_appendix.tex
\section{Notation}
For any quantity that depends on population statistics, $A$, we use $\hat A$ to denote the corresponding sample estimate. For a vector $a\in \mathbb{R}^n$, we use $a^{(\dot k)}\in \mathbb{R}^{n^2+n}$ to denote the vector consisting of all monomials of degree {\em less than or equal to} $k$ with variables in $a$. Thus, we have
$a^{(\dot 2)} = (1, a_1, a_2, \cdots, a_n, a_1^2, a_1 a_2, \cdots, a_n^2)^{\top}$. 
We use $a^{(k)}\in \mathbb{R}^{n^2+n}$ (that is, with no dot) to denote the vector consisting of all monomials of {\em exactly} degree $k$ with variables in $a$, so $a^{(2)} = (a_1^2, a_1a_2, a_1a_3, \cdots a_n^2)^{\top}$. For a matrix $A$, we use $r(A)$ to denote the vector formed by flattening the matrix, so $r(A) = (A_{11}, A_{12}, A_{13}, \cdots, A_{n-1, n}, A_{nn})^{\top}$.
\section{Auxillary proofs}
\subsection{General Results}
\begin{lemma}\label{lemma:vectorHoefding}
Consider a sequence of i.i.d. random vectors $V_1, V_2, \cdots, V_m\in \mathbb{R}^d$ with mean $\mu$ such that $\Vert V_i\Vert_\infty < L$ always. Then for $m>\frac{dL^2}{\epsilon^2}\log\frac{d}{\delta}$, we have $$\Pr[\Vert \frac{1}{m}\sum_{i=1}^m V_i-\mu\Vert < \epsilon] >1-\delta.$$
\end{lemma}
\begin{proof}
By a standard Hoeffding bound, with probability $1-\delta$, each entry of $\frac{1}{m}\sum_{i=1}^m V_i$ will be within $\frac{\epsilon}{\sqrt{d}}$ of its mean for  $m\ge\frac{dL^2}{\epsilon^2}\log\frac{d}{\delta}$. Then the squared norm of the error $\frac{1}{m}\sum_{i=1}^mV_i-\mu$ will be less than $\epsilon$.
\end{proof}
\begin{lemma}[\cite{vershynin2011introduction}]\label{lemma:coveringnumbersphere}
The covering number of the Euclidean ball of radius $R$ is $(1+\frac{2R}{\epsilon})^d$.
\end{lemma}
\begin{lemma}\label{lemma:normsquare}
Let $A\in \mathbb{R}^{n\times n}$ be a matrix and $x\in \mathbb{R}^n$ be a column vector, then we have$\Vert Ax\Vert_2^2 = r(A^\top A)x^{(2)}$.
\end{lemma}
\begin{proof}
This follows by writing the norm sqaured as a quadratic form.
\end{proof}
\begin{lemma}\label{ravelasmatrix}
Consider a diagonal matrix $\Phi \in \mathbb{R}^{n \times n}$ . Then there exists a fixed matrix $V \in \mathbb{R}^{n^2 \times n}$ such that $r(\Phi) = V \text{diag}(\Phi)$.
\end{lemma}
\begin{proof}
To construct $V$, we take the $n \times n$ identity matrix and insert with $n-1$ rows of all $0$s between every row of the identity matrix obtain $V \in \mathbb{R}^{n^2 \times n}$ which looks as follows:
$$V=
\begin{pmatrix}
1&0&0\cdots&0\\
0&0&0\cdots&0\\
\vdots\\
0&1&0\cdots&0\\
0&0&0\cdots&0\\
\vdots\\
0&0&0\cdots&0\\
0&0&0\cdots&1
\end{pmatrix}
$$.
\end{proof}   
\begin{lemma}[\cite{sullivant2010trek}, \cite{uhler2013geometry}]\label{lemma:inverseexpansion}
For the class of linear SEMs defined in Equation \ref{eq:SEM}, we have the following
\begin{itemize}
\item $\Sigma_\theta = J(1-B_\theta)^{-1}\Xi_\theta\left((1-B_\theta)^{-1}\right)^\top J^\top$ where $\Xi_\theta = \E_\theta[\eta\eta^\top]$ and $J$ is identity concatenated with one column of zeros at the index corresponding to $v_t$. 
\item $(I-B_\theta)^{-1} = \sum_{k=0}^{n-1}B_\theta^k$
\item $\Sigma_\theta = \sum_{k = 0}^{2n-2} \sum_{\substack{r+s = k\\r, s < n-1}} B^r\Xi_\theta(B^\top)^s.$
\end{itemize}
\end{lemma}
\begin{proof}
Note that the SEM equations can all be combined into a single vector equation of the form 
$$X = B_\theta X+\eta\implies X=(I-B_{\theta})^{-1}\eta$$
for some lower triangular matrix $B_\theta\in \mathbb{R}^{n\times n}$, and a noise vector $\eta$. We can then write the covariance of the features $X_{-t}$, $\Sigma_{\theta}$, in terms of $B_{\theta}$. Let $J\in \mathbb{R}^{n\times n+1}$ denote the matrix that is identity concatenated with an additional column consisting of all zeros in the index corresponding to $v_t$, such that it selects the submatrix of $\E_\theta[XX^\top]$ corresponding to $\Sigma_\theta = \E_\theta[X_{-t}X_{-t}^\top]$
\begin{align*}
\Sigma_{\theta} &:= J\E_{\theta}[XX^\top] J^\top\\
&= J(I-B_{\theta})^{-1}\E_{\theta}[\eta \eta^\top]\left((I-B_{\theta})^{-1}\right)^\top J^\top
\end{align*}

Since $\Vert B_\theta\Vert < 1$, we have $(I-B_\theta)^{-1} = \sum_{k=0}^{\infty}B_\theta^k$. Because $B_\theta$ is lower triangular, $B_\theta^k = 0$ for $k\ge n$. So $(I-B_\theta)^{-1} = \sum_{k=0}^{n-1}B_\theta^k$. Similarly $((I-B_\theta)^{-1})^\top = \sum_{k=0}^{n-1}(B_\theta^\top)^k$. The result follows by considering every cross term in the product of the sums.
$$\Sigma_\theta = \sum_{r = 0}^{n-1} (B_\theta)^r \Xi_\theta\sum_{s = 0}^{n-1} (B^\top_\theta)^s = \sum_{k = 0}^{2n-2} \sum_{\substack{r+s = k\\r, s < n-1}} B^r\Xi_\theta(B^\top)^s.$$
\end{proof}

\section{Invariant Representations}
Consider the SEM model of Section \ref{section:model}.
\begin{align*}
X_i &= \left(\beta_i^{\theta}\right)^\top \Pa_i+\eta_i\\
X_t &= \beta^\top \Pa_t + \eta_t 
\end{align*}
We see that a representation that projects onto the parents is indeed invariant, $\E_\theta[X_t|\Pa_t] = \beta^\top \Pa_t$ and the RHS has no dependence on $\theta$. If $X$ is independent of $\eta_t$, actually more can be said: the identity representation itself is invariant! 
\begin{lemma}\label{lemma:ERMisIRM}
If $X\indep \eta_t$, then $I\in \I(\Theta)$.
\end{lemma}
\begin{proof}
We begin with Equation (\ref{eq:gradientsimplified})
$$\nabla \R^\theta_\Phi(\fc) = \Phi^\top \Sigma_{\theta}\Phi\fc - \Phi^\top \E_{\theta}[X_{\theta}X_t^\top]$$
Plugging in $\Phi = I$, we have
\begin{align*}
\nabla \R^\theta_I(\fc) 
&= \Sigma_{\theta}\fc - \E_{\theta}[X_{\theta}X_t^\top]\\
&= \Sigma_{\theta}\fc -  \E_{\theta}[X\left(\Pa_t^\top\beta+\eta_t\right)]\\
&= \Sigma_{\theta}\fc - \E_{\theta}[X\left(X^\top\overline \beta+\eta_t\right)]\\
&= \Sigma_{\theta}\fc -  \Sigma_{\theta}\beta+\E_\theta[X\eta_t]\\
&=  \Sigma_{\theta}\fc -  \Sigma_{\theta}\beta+\E_{\theta}[X]\E_\theta[\eta_t]\\
&= \Sigma_{\theta}(\fc - \beta)
\end{align*}
So the loss is minimized at the common $\fc = \beta$ (that is, independent of intervention).
\end{proof}
\begin{remark}
In the absence of anti-causal variables (variables other than the target variable $X_t$ that are ``downstream" from $X_t$), the condition of Lemma \ref{lemma:ERMisIRM} is satisfied, and we have the convenient result that the ERM solutions are actually already invariant.
\end{remark}
\subsection{Representations}
In this section, we will look more closely at the set of representations.

\begin{lemma}\label{lemma:headdoesntmatter}
Any approximately invariant solution $(\Phi, f)$, corresponds to another invariant solution $(\Phi', \fc)$ where $\Phi' = \Phi A^{-1}$ for some invertible $A$, and $\fc = (1,0, \cdots, 0)$.
\end{lemma}
\begin{proof}
We will construct $A$ as follows. For the first column, take $A_{:, 0} = f$. Then use Gram-Schmidt to complete $A$, iteratively using columns orthogonal to previous columns. This ensures that $A \fc = f$, while also keeping $A$ invertible. Furthermore, we have that $\Vert \nabla \R^\theta_\Phi(f)\Vert \le \epsilon\iff \Vert \nabla \R^\theta_{\Phi A}(A^{-1}f)\Vert \le \epsilon.$
\end{proof}
Also, since the least squares objective is strongly convex, using gradients instead of loss values does not change the set of invariant representations
\begin{lemma}\label{lemma:lossisgrad}
$$\fc = \argmin \R^\theta_\Phi(f)\hspace{0.3cm} \forall \theta\in\Theta \iff \nabla \R^\theta_\Phi(\fc) = 0 \hspace{0.3cm} \forall\theta\in \Theta$$
\end{lemma}
\begin{proof}
Follows from strong convexity of the loss function.
\end{proof}
\begin{lemma}\label{lemma:gradientexpanded}
$$\Vert \nabla \R^\theta_\Phi(\fc)\Vert = \Vert\Phi^\top \Sigma_\theta\Phi\fc - \Phi^\top (\Sigma_\theta\beta-J(1-B_\theta)^{-1}e_t)\Vert$$
\end{lemma}
\begin{proof}
\begin{align}
&\Vert \nabla \E[\left(f^\top \Phi^\top X_{-t} - X_t\right)^2\big|_{f = f_\circ}\Vert_2\nonumber\\
&=2\Vert \E[ \Phi^\top X_{-t}\left(X_{-t}^\top \Phi \fc - X_t\right)]\Vert \\
&= 2\Vert \Phi^\top \Sigma_{\theta}\Phi\fc - \Phi^\top \E_{\theta}[X_{-t}X_t^\top]\Vert\nonumber\\
&= 2\Vert \Phi^\top \Sigma_{\theta}(\Phi\fc-\overline\beta_t)-\Phi^\top\E_{\theta}[X_{-t}^\top\eta_t]\Vert\nonumber\\ 
&=2\Vert \Phi^\top \Sigma_\theta(\Phi\fc - \beta)-\Phi^\top \E_\theta[J(1-B_\theta)^{-1}\eta\eta_t]\Vert\nonumber\\
&=2\Vert\Phi^\top \Sigma_\theta(\Phi\fc - \beta)-\Phi^\top J(1-B_\theta)^{-1}\E_\theta[\eta\eta_t]\Vert\nonumber\\
&=2\Vert\Phi^\top \Sigma_\theta\Phi\fc - \Phi^\top (\Sigma_\theta\beta-J(1-B_\theta)^{-1}e_t)\Vert
\end{align}
\end{proof}

\begin{lemma}\label{appendix_lemma:invariance_to_classification}
There is a function class $\F$ mapping interventions $\theta\in \Theta$ to $\{0, 1\}$, such that
$$\Phi\in \I_{\fc}^\epsilon(\ThetaTr) \iff f^\epsilon_\Phi(\theta) = 1 ~\forall \theta\in \ThetaTr$$ 
\end{lemma}
\begin{proof}

By Lemmas \ref{lemma:gradientexpanded} and \ref{lemma:inverseexpansion}, the invariance condition can be simplified as follows (for approriate $J, e_t$ defined in the Lemma):
\begin{equation}
\Vert \nabla \E[\left(f^\top \Phi^\top X_{-t} - X_t\right)^2\big|_{f_\circ}\Vert_2=2\Vert\Phi^\top \Sigma_\theta\Phi\fc - \Phi^\top (\Sigma_\theta\beta-J(1-B_\theta)^{-1}e_t)\Vert\le \epsilon.\label{eq:gradientsimplified}
\end{equation}
Our desired function class is $\F = \{f^\epsilon_\Phi(\theta)\}$ parameterized by $\Phi$ where  
\begin{equation*}
f^\epsilon_\Phi(\theta):=
\mathbbm{1}\{2\Vert\Phi^\top \Sigma_\theta\Phi\fc - \Phi^\top (\Sigma_\theta\beta-J(1-B_\theta)^{-1}e_t)\Vert \le \epsilon\},
\end{equation*}
is a function mapping an intervention to $\{0, 1\}$ such that
$\Phi\in \I^\epsilon_{\fc}(\theta) \iff f^\epsilon_\Phi(\theta) = 1.$
\end{proof}


\begin{corollary}\label{appendix_cor:VCtoPAC}
If the VC dimension of $\F$ is bounded by $d$, then given at least $O(\frac{d+\log\frac{1}{\delta}}{\delta'})$ training interventions, if $\I^\epsilon(\Theta)$ is not empty, we have that with probability $1-\delta$ over the set of training interventions, 
$$\Pr_{\theta \sim \D} (\I^\epsilon_{\fc}(\ThetaTr)\not\subset \I^\epsilon_{\fc}(\theta)) \leq \delta'. $$
\end{corollary}
\begin{proof}
Let $\tilde{D}$ be a distribution over $\Theta \times \{\pm1\}$ such that the marginal over $\Theta$ is $\D$, and the marginal over $\{\pm 1\}$ is the dirac-delta on the element '1' (i.e. it always takes value $1$). We can now apply Lemma \ref{lemma:PAC} to our function class $\F=\{f^{\epsilon}_{\Phi}(\theta)\}$ and distribution $\tilde{\D}$. By assumption, there exists a representation $\Phi\in \I^\epsilon_{\fc}(\Theta) = \bigcap_{\theta\in \ThetaTr} \I^\epsilon(\theta)$, which means $f^\epsilon_{\Phi}(\theta) = 1$ for all $\theta\in \Theta.$ 
Since $\ThetaTr\subseteq\Theta$, we have $f^\epsilon_{\Phi}(\theta) = 1$ for all $\theta\in \ThetaTr.$
PAC learnability now states that given $O(\frac{d+\log\frac{1}{\delta}}{\delta'})$ interventions, if $\tilde \Phi$ is such that $\tilde \Phi\in f^\epsilon_\Phi(\ThetaTr)$ (so that it is a hypothesis that fits the training data exactly), we have that with probability at least $1-\delta$, $\Pr_{\theta\sim \D}[f_{\tilde \Phi}^\epsilon(\theta) = 1] >1-\delta'.$
\end{proof}

\section{Intervention Models}\label{section:intervention_models_appendix}
\subsection{General Interventions}
\begin{lemma}\label{lemma:general}
For some matrix $U_\theta$, we have $$\Phi^\top \Sigma_\theta\Phi\fc - \Phi^\top (\Sigma_\theta\beta-J(1-B_\theta)^{-1}e_t) = U_\theta r(\Phi)^{(2)}.$$
\end{lemma}
\begin{proof}

We will see that both of these terms can be written as $U_\theta r(\Phi)^{(\dot 2)}$ for some (different) choices of $U_\theta$. We will enumerate the coordinates of $r(\Phi)^{(\dot 2)}_{(ij),(kl)}$ as $\Phi_{ij}\Phi_{kl}$. Note that $r(\Phi)^{(\dot 2)}$ also contains the individual $\Phi_{ij}$, and we will use $\cdot$ to indicate this extra alphabet. To clarify, some of the coordinates are then of the form $r(\Phi)^{(\dot 2)}_{(\cdot, (kl))} = \Phi_{kl}$. Now observe that
$$\Phi^\top \Sigma_\theta\Phi\fc = U^1_\theta r(\Phi)^{(\dot 2)}$$
where $$(U^1_\theta)_{a, ((ij),(kl))} = 
\begin{cases}
\Sigma_{jk}(\fc)_l, & a = i\text{ and }(ij)\ne \cdot\text{ and } (kl)\ne \cdot\\
0, & a\ne i \text{ or } (ij) = \cdot \text{ or } (kl)= \cdot
\end{cases}
$$
The second term is almost in the form we would like already
$$\Phi^\top (\Sigma_\theta\beta-J(1-B_\theta)^{-1}e_t) = U^2_\theta r(\Phi)^{(\dot 2)}$$
where $$(U^2_\theta)_{a, ((ij),(kl))} = 
\begin{cases}
(\Sigma_\theta\beta-J(1-B)^{-1}e_t)_{ij}&  (ij)\ne \cdot\text{ and } (kl)= \cdot\\
0, & (ij)= \cdot\text{ or } (kl)\ne \cdot\\
\end{cases}
$$
Our result follows by setting $U_\theta  = U_\theta ^1+U_\theta ^2$.
\end{proof}

\begin{theorem}

Suppose that we are given $O(\frac{n^{4}+\log\frac{1}{\delta}}{\delta'})$ interventions drawn independently from distribution $\D$ over the intervention index set $\Theta$. Then with probability at least $1-\delta$ over the randomness in $\ThetaTr$, the following statement holds: 
$$\Pr_{\theta \sim \D} (\I^\epsilon_{\fc}(\ThetaTr)\notin \I^\epsilon_{\fc}(\theta)) \leq \delta'. $$


In other words, an $\epsilon$- approximate invariant representation on the training environments generalizes with high probability to the test environment.
\end{theorem}
\begin{proof}
We again begin with our expression for the gradient of the loss given by Lemma \ref{lemma:gradientexpanded}.
$$\nabla \R^\theta_\Phi(\fc) = \Phi^\top \Sigma_\theta\Phi\fc - \Phi^\top (\Sigma_\theta\beta-J(1-B_\theta)^{-1}e_t).$$
We show in Lemma \ref{lemma:general} that we can write this as 
$\nabla \R^\theta_\Phi(\fc) = U_\theta r(\Phi)^{(\dot2)}$ for some matrix $U_\theta $. From Lemma \ref{lemma:normsquare} we can write the squared norm as 
$\Vert \nabla \R^\theta_\Phi(\fc)\Vert_2^2 = r(U_\theta ^\top U_\theta )(r(\Phi)^{(\dot 2)})^{(2)}.$
Finally, we can now write our function class as a class of halfspace classifiers.
$$f^\epsilon_\Phi(\theta)= 1\iff r(U_\theta ^\top U_\theta )(r(\Phi)^{(\dot 2)})^{(2)}\le \epsilon.$$
Now consider the map $\Psi: \text{diag}(\Phi^{(\dot 4)})\to r(\Phi)^{(\dot 2)(2)}$. Because $\Phi$ is diagonal, this is well defined. Each term in $r(\Phi)^{(\dot 2)(2)}$ contains at most $4$ terms of $\Phi$, and each of these is an element of $\Phi^{(\dot 4)}$. Also, since the relation between entries of $r(\Phi)^{(\dot 2)(2)}$ and entries of $\text{diag}(\Phi^{(\dot 4)})$ is fixed, $\Psi$ is actually a linear transformation (that is, it is some sparse $0/1-$matrix). We can then write  
$$f^\epsilon_\Phi(\theta)= 1\iff r(U_\theta ^\top U_\theta )\Psi \text{diag}(\Phi^{(\dot 4)})\le \epsilon.$$
Finally, note that $\text{diag}(\Phi^{(\dot 4)})\subset \mathbb{R}^{(n+1)^4}$, that is, we have written the function class $\F$ as corresponding to some \textit{subset} of all halfspace classifiers in $\mathbb{R}^{(n+1)^4}$. By Lemma \ref{lemma:halfspaceVC}, we know that the VC dimension of $\F$ is bounded by $(n+1)^4$. The result follows from Corollary \ref{cor:VCtoPAC}.
\end{proof}

\subsection{Atomic interventions}
Here we consider the instance in which our family of interventions consists of atomic interventions on a total of $k$ nodes $\At$. An atomic intervention at node $i$ replaces the generative equation $X_i = \gamma_i(e)^\top \Pa_i + \eta_i$ by the \textit{assignment} $X_i = a_i$ for some scalar $a_i$. Note that we can interpret this as zeroing out the corresponding entries of $B$ and changing the noise variable to be a constant $a_i$. 

\begin{lemma}\label{lemma:atomic}
For some matrix $U$, we have $$\Vert \Phi^\top \Sigma_\theta\Phi\fc - \Phi^\top (\Sigma_\theta\beta-J(1-B_\theta)^{-1}e_t)\Vert_2^2 = r(\Phi)^{(\dot 4)} U a^{(\dot 2)}.$$
\end{lemma}
\begin{proof}

We will examine the dependence of $\E_{\theta}[X_{-t}X_{-t}^\top]$ and $\E_{\theta}[X_{-t}X_t^\top]$ on $a^i_\theta$ starting from Lemma \ref{lemma:inverseexpansion}.

For the former, note that $(\Sigma_{\theta})_{ij}$ can be interpreted as having one term corresponding to every path from $i$ \textit{backwards} to some node $k$ and then forwards to $j$. Because atomic interventions essentially cut off the connections between a node and its parents, any path that consists initially of backwards edges and then forward edges can only have a quadratic dependence on the parameters of the intervention, regardless of how many nodes are intervened on. It is only such paths that contribute to the covariance matrix, and thus to the thresholded polynomials that determine approximate invariance. We will make this explicit below.

The nature of atomic interventions is that $B_\theta = B_{\Theta} = BI_{\Theta}$ for each of the interventions $\theta\in \Theta$, where $I_\Theta$ is an identity matrix with zeros on the diagonal entries corresponding to $\At$ (we are disconnecting each of the sites of intervention from their parents). In addition, the noise covariance changes. Previously, with independent unit variance noise, the covariance was $\Xi_{\theta} = \E_{\theta}[\eta\eta^\top] = I$. Since we are incorporating the assignments for the atomic interventions into the noise, it is now given by
$$(\Xi_{\theta})_{ij} = 
\begin{cases}
1, &i=j\not\in \At\\
0, &i\in \At, j\not\in \At, i\ne j\\
0, &i\not\in \At, j\in \At,  i\ne j\\
a_i^\theta a_j^\theta & i, j\in \At
\end{cases}
$$
That is, we have replaced the submatrix corresponding to the sites of the interventions with $a^\theta (a^\theta)^\top$ where $a^\theta$ denotes the vector of assigments $a^\theta = (a^\theta_1, a^\theta_2, \cdots, a^\theta_{|\At|})$. 

\begin{align*}
    (\Phi^\top \Sigma_\theta\Phi\fc)_{j}
    &=\sum_{k=0}^{2n-2}\sum_{\substack{r+s = k\\r,s<n}}(\Phi^\top JB_\Theta^r\Xi_{\theta}(B_\Theta^\top)^sJ^\top\Phi\fc)_{j}\\
    &= \sum_{k=0}^{2n-2}\sum_{\substack{r+s = k\\r,s<n}}\sum_{j_1, j_2, j_3, j_4, j_5, j_6, j_7}\Phi_{j_1, j}J_{j_1, j_2}(B_\Theta^r)_{j_2, j_3}(\Xi_{\theta})_{j_3, j_4}\left(B_\Theta^s\right)_{j_5, j_4}J^\top_{j_5, j_6}\Phi_{j_6, j_7}(\fc)_{j_7}\\
    &= \sum_{k=0}^{2n-2}\sum_{\substack{r+s = k\\r,s<n}}\sum_{\substack{j_1, j_2, j_3, j_5, j_6, j_7\\ j_2\not\in \At}}\Phi_{j_1, j}J_{j_1, j_2}(B_\Theta^r)_{j_2, j_3}(\Xi_{\theta})_{j_3, j_3}\left(B_\Theta^s\right)_{j_5, j_4}J^\top_{j_5, j_6}\Phi_{j_6, j_7}(\fc)_{j_7}+\\
    &\hspace{1cm}\sum_{\substack{j_1, j_2, j_3, j_4, j_5, j_6, j_7\\ j_2, j_3\in \At}}\Phi_{j_1, j}J_{j_1, j_2}(B_\Theta^r)_{j_2, j_3}(\Xi_{\theta})_{j_3, j_4}\left(B_\Theta^s\right)_{j_5, j_4}J^\top_{j_5, j_6}\Phi_{j_6, j_7}(\fc)_{j_7}\\
    &= \sum_{k=0}^{2n-2}\sum_{\substack{r+s = k\\r,s<n}}\sum_{\substack{j_1, j_2, j_3, j_5, j_6, j_7\\ j_2\not\in \At}}\Phi_{j_1, j}J_{j_1, j_2}(B_\theta^r)_{j_2, j_3}\left(B_\Theta^s\right)_{j_5, j_4}J^\top_{j_5, j_6}\Phi_{j_6, j_7}(\fc)_{j_7}+\\
    &\hspace{1cm}\sum_{\substack{j_1, j_2, j_3, j_4, j_5, j_6, j_7\\ j_2, j_3\in \At}}\Phi_{j_1, j}J_{j_1, j_2}(B_\Theta^r)_{j_2, j_3}a_{j_3}^{\theta}a_{j_4}^{\theta}\left(B_\Theta^s\right)_{j_5, j_4}J^\top_{j_5, j_6}\Phi_{j_6, j_7}(\fc)_{j_7}\\
    &=\sum_{k=0}^{2n-2}\sum_{\substack{r+s = k\\r,s<n}}r(\Phi)^{(2)}U^{1, \Theta, r, s}_ja^{(2)}\\
    &=r(\Phi)^{(2)}U^{1, \Theta}_ja^{(2)}\\
\end{align*}
Stacking the $U_j^{1, \Theta}$ together gives us $\Phi^\top \Sigma_\theta\Phi\fc = r(\Phi)^{(2)}U^{1, \Theta}a^{(2)}$.

Similarly to Lemma \ref{lemma:general}, we have
\begin{align*}
\Phi^\top \Sigma_\theta\beta
&=\sum_{k=0}^{2n-2}\sum_{\substack{r+s = k\\r,s<n}}\sum_{j_1, j_2, j_3, j_4}\Phi_{j_1, j}(B_\theta^r)_{j_1, j_2}(\Xi_{\theta})_{j_2, j_3}(B_\theta^s)_{j_4, j_3}\overline \beta_{j_4}\\
&=\sum_{k=0}^{2n-2}\sum_{\substack{r+s = k\\r,s<n}}\sum_{\substack{j_1, j_2, j_4\\j_2\not\in \At}}\Phi_{j_1, j}(B_\theta^r)_{j_1, j_2}(\Xi_{\theta})_{j_2, j_3}(B_\theta^s)_{j_4, j_3}\overline \beta_{j_4}+\sum_{\substack{j_1, j_2, j_3, j_4\\j_2, j_3\in \At}}\Phi_{j_1, j}(B_\theta^r)_{j_1, j_2}(\Xi_{\theta})_{j_2, j_3}(B_\theta^s)_{j_4, j_3}\overline \beta_{j_4}\\
&=\sum_{k=0}^{2n-2}\sum_{\substack{r+s = k\\r,s<n}}\sum_{\substack{j_1, j_2, j_4\\j_2\not\in \At}}\Phi_{j_1, j}(B_\theta^r)_{j_1, j_2}(B_\theta^s)_{j_4, j_3}\overline \beta_{j_4}+\sum_{\substack{j_1, j_2, j_3, j_4\\j_2, j_3\in \At}}\Phi_{j_1, j}(B_\theta^r)_{j_1, j_2}a_{j_2}a_{j_3}(B_\theta^s)_{j_4, j_3}\overline \beta_{j_4}\\
&=\sum_{k=0}^{2n-2}\sum_{\substack{r+s = k\\r,s<n}}r(\Phi)^{(2)}U^{2,r,s}_{j}a^{(2)}\\
&=r(\Phi)^{(2)}U^{2}_{j}a^{(2)}\\
\end{align*}
Stacking the $U^{2}_{j}$ together, we get $\Phi^\top \Sigma_\theta\overline\gamma^\top = U^{2}a^{(2)}$. 
Now, 
\begin{align*}
\Phi^\top J(1-B_\theta)^{-1}e_t=r(\Phi)^{(2)}U^{3, \Theta}a^{(\dot 2)}
\end{align*}
simply by noting that there is no dependence on the parameters $a_i$. Taking the difference, we have for some $U^{4, \Theta}$,
$$\Phi^\top \Sigma_\theta\Phi\fc - \Phi^\top (\Sigma_\theta\beta-J(1-B_\theta)^{-1}e_t) = r(\Phi)^{(2)}U^{4, \Theta}a^{(2)}.$$
Finally, the squared norm consists of terms that contain up to two terms of $r(\Phi)^{(\dot 2)}$ and two terms of $a^{(\dot 2)}$, this can be written as 
$$ \Vert \Phi^\top \Sigma_\theta\Phi\fc - \Phi^\top (\Sigma_\theta\beta-J(1-B_\theta)^{-1}e_t)\Vert_2^2 = r(\Phi)^{(\dot 4)}U^{\Theta}a^{(\dot 4)}.$$

\end{proof}

\begin{theorem}\label{thm:hard}
Given $|\ThetaTr| = O(\frac{k^4+\log\frac{1}{\delta}}{\delta'})$ interventions drawn independently from a distribution $\D$ over all atomic interventions on some fixed set $\At$ of $k$ nodes that with probability at least $1-\delta$ over the randomness in $\ThetaTr$, with probability at least $1-\delta'$ over the randomness in $D$, we have for $\theta\sim \D$,
$\I^\epsilon(\ThetaTr) \subseteq \I^\epsilon(\theta)$
\end{theorem}
\begin{proof}
Recall Equation (\ref{lemma:gradientexpanded}):
$$\Vert \R^\theta_\Phi(\fc)\Vert_2\nonumber
=2\Vert\Phi^\top \Sigma_\theta\Phi\fc - \Phi^\top (\Sigma_\theta\beta-J(1-B_\Theta)^{-1}e_t)\Vert$$
We show in Lemma \ref{lemma:atomic} that we can simplify this expression down to one of the form 
$$\Vert \nabla \R^\theta_\Phi(\fc)\Vert_2^2 = r(\Phi)^{(\dot 4)}U^\Theta a^{(4)}$$
for some different matrix $U^\theta$. The function class $\F$ now consists of half-spaces, $$f^\epsilon_\Phi(\theta) = 1 \iff r(\Phi)^{(\dot 4)}U^\Theta a^{(4)}\le \epsilon$$

We can thus re-parameterize $\F$ using $r(\Phi)^{(4)}$ rather than $\Phi$. Since these half-spaces live in $\mathbb{R}^{k^4}$ rather than $\mathbb{R}^{n^2}$, by Lemma \ref{lemma:halfspaceVC} the VC-dimension is bounded by $k^4$. From Corollary \ref{cor:VCtoPAC}, we have that $O((k^4+\log\frac{1}{\delta})/\delta')$ interventions suffices to ensure that any $\tilde \Phi$ satisfying $f^\epsilon_{\tilde \Phi}(\theta) = 1$ for all $\theta\in \ThetaTr$ satisfies $f^\epsilon_\Phi(\theta) = 1$ with probability $1-\delta'$ for $\theta\sim \D$.
\end{proof}
\subsection{Soft Interventions}

Soft interventions are those in which the \textit{weights} of the SEM are modified while the underlying causal structure remains the same. That is, a soft intervention at node $v_i$ is equivalent to replacing the equation $X_i = (\beta_i^{\theta_\circ})^\top X + \eta_i$ with $X_i = (\beta_i^\theta)^\top X + \eta_i$. Note that this is equivalent also to changing the $i$th row of $B_{\theta}$. 
$\beta\in \mathbb{R}^{\sum_{i\in \At} |\Pa_i|}$ denote the vector of all weights involved in the set of soft interventions we are considering. Let $\underline \beta^{(2k)}$ denote the vector consisting of every combination of up to $2k$ entries of $\beta$ modulo terms of the form $\beta_{l_2, l_1}^\theta\beta_{l_3, l_1}^\theta\beta_{l_4, l_1}^\theta$. In other words, $\underline\beta^{(2k)}$ does not contain any three weights from the same site.
\begin{lemma}\label{lemma:soft}
For some matrix $U$, we have $$\Vert \Phi^\top \Sigma_\theta\Phi\fc - \Phi^\top (\Sigma_\theta\beta-J(1-B_\theta)^{-1}e_t)\Vert^2 = r(\Phi)^{(\dot 4)} U \underline \beta^{(2k)(\dot 2)}.$$
\end{lemma}
\begin{proof}
We will look at how $\Sigma_\theta$ and $(1-B_\theta)^{-1}$ depend on the parameters of the intervention. 

Denote by $\Gamma_{ki}$ the set of directed paths from node $v_k$ to $v_i$ in G. From Lemma \ref{lemma:inverseexpansion} we have
\begin{align*}
(\Sigma_{\theta})_{ij} &= 
\sum_{\substack{\gamma_{ki}\in \Gamma_{ki}\\ \gamma_{kj}\in \Gamma_{kj}}} \prod_{\substack{(l_1, l_2)\in \gamma_{ki}\\l_2\in \At}} \beta^{\theta}_{l_1, l_2}\prod_{\substack{(l_1, l_2)\in \gamma_{ki}\\l_2\not\in\At}} \beta_{l_1, l_2}\\
&\hspace{1cm}\prod_{\substack{(l_1', l_2')\in \gamma_{kj}\\l_2'\in \At}} \beta^{\theta}_{l_1', l_2'}\prod_{\substack{(l_1', l_2')\in \gamma_{ki}\\l_2' \not\in \At}} \beta_{l_1, l_2}
\end{align*}
We can bound the number of terms in this expansion that depend on the intervention using a simple counting argument. There are no terms that are multiples of $\beta_{l_2, l_1}^\theta\beta_{l_3, l_1}^\theta\beta_{l_4, l_1}^\theta$ since there can be no directed path that passes through a single node twice (else we could find a cycle in the acyclic graph $G$). Thus each term in $(\Sigma_{\theta})_{ij}$ consists of at most two terms from each site of intervention. There are $k$ total sites, and $d^2$ choices of two terms at each, for a total of $(d^2)^k$ distinct possible terms in the covariance. 

So each entry in the covariance matrix is an inner product of some vector with $\underline \beta^{(2k)}$.

Similarly, consider $(1-B_\theta)^{-1}$. We can write this as $\sum_{k=0}^{n-1}B_\theta^r$ (since $B_\theta$ is lower triangular, higher order terms are 0). We have
\begin{equation}
((1-B_\theta)^{-1})_{ij}=\sum_{\gamma_{ki}\in \Gamma_{ki}} \prod_{\substack{(l_1, l_2)\in \gamma_{ki}\\l_2\in \At}} \beta^{\theta}_{l_1, l_2}\prod_{\substack{(l_1, l_2)\in \gamma_{ki}\\l_2\not\in\At}} \beta_{l_1, l_2}
\end{equation}
Similarly, each entry here is an inner product of some vector with $\underline \beta^{(k)}$. In this case, no terms of the form $\beta_{l_2, l_1}^\theta\beta_{l_3, l_1}^\theta$ occur, since each backwards path can only pass through each site of intervention once.

We can then argue analogously to Lemma \ref{lemma:atomic} that the complete each term in the gradient of the loss can be written using terms consisting of elements of $\underline \beta^{(2k)}$ and elements of $r(\Phi)^{(\dot 2)}$. The squared norm of the gradient can then be written using terms consisting of pairs of $\underline \beta^{(2k)}$ and pairs of elements of $r(\Phi)^{(\dot 2)}$. In other words,
$$\Vert \Phi^\top \Sigma_\theta\Phi\fc - \Phi^\top (\Sigma_\theta\beta-J(1-B_\theta)^{-1}e_t)\Vert^2 = r(\Phi)^{(\dot 4)} U \underline \beta^{(2k)(\dot 2)}$$

\end{proof}
\begin{theorem}\label{thm:soft}
Given $O(\frac{d^{4k}+\log\frac{1}{\delta}}{\delta'})$ interventions drawn independently from a distribution $\D$ over all soft interventions on some fixed set $\At$ of $k$ nodes such that each intervened node has in-degree at most $d$, we have that any $\epsilon-$approximately invariant representation $\Phi$ satisfies that with probability at least $1-\delta$ over the randomness in $\ThetaTr$, with probability at least $1-\delta'$ over the randomness in $D$, we have for $\theta\sim \D$,$\I^\epsilon(\ThetaTr) = \I^\epsilon(\theta)$
\end{theorem}
\begin{proof}
We show in the Lemma \ref{lemma:soft} that we can now write the norm of the gradient of the loss in terms of a product in a lower dimensional space again, similar to what we did previously.
$$\Vert\nabla \R^\theta_\Phi(\fc)\Vert_2^2 = r(\Phi)^{(4)}V\underline{\beta}^{(4k)}.$$
to get an upper bound on the VC dimension of $d^{4k}$. The theorem follows from Corollary \ref{cor:VCtoPAC}.
\end{proof}
\subsection{Indirect Observations}\label{appendix:linear_indirect_observations}
For this section, we change the observation model. There is assumed to be an underlying linear SEM, defined as in \ref{eq:SEM}, but we observe $Z = SX$ for some linear $S$. Here $Z$ is a vector, and there is a specific element of $Z$, say $Z_t$, that we consider to be the target variable. Note that this could just be $X_t$ when the row of $S$ corresponding to $Z_t$ is one-hot. \footnote{Note that the only reason we assume linearity anywhere (in the generation of the SEM or in the observation model) is because this warrants the use of linear regression to find models.}

We use the same definition for the invariant representations, namely,
$$\Phi\in \I^\epsilon_{\fc}(\theta) \iff \Vert \nabla \E_\theta[f^\top \Phi Z_{-t} - Z_t] \mid_{\fc}\Vert_2\le \epsilon.$$

Here lemma \ref{lemma:gradientexpanded} should be modified to account for the indirect observations. 

\begin{lemma}\label{lemma:indirect_gradient_expansion}
Take $S_{-t}$ to be the matrix formed by the rows of $S$ excluding the one corresponding to $Z_t$, and $S_t$ to be just the row corresponding to $Z_t$. Then we have,
$$\Vert \nabla \E_\theta[f^\top \Phi Z_{-t} - Z_t] \mid_{\fc}\Vert_2 = \Vert \Phi^\top S_{-t}\Sigma_\theta \big(S_{-t}^\top\Phi {\fc} - \overline \beta_t\big) - \Phi^\top S_{-t} J(1-B_\theta)^{-1}e_t S_{t}^\top \Vert_2.$$
\end{lemma}
\begin{proof}
\begin{align*}
\Vert \nabla \E_\theta[f^\top \Phi Z_{-t} - Z_t] \mid_{\fc}\Vert_2 
&= \Vert \Phi^\top \E_\theta[Z_{-t}Z_{-t}^\top]\Phi {\fc} - \Phi^\top \E_\theta [Z_{-t}Z_t] \Vert_2\\
&= \Vert \Phi^\top \E_\theta[S_{-t}X_{-t}X_{-t}^\top S_{-t}^\top]\Phi {\fc} - \Phi^\top \E_\theta [S_{-t}X_{-t}X_t S_{t}^\top] \Vert_2\\
&= \Vert \Phi^\top S_{-t}\Sigma_\theta S_{-t}^\top\Phi {\fc} - \Phi^\top S_{-t} \E_\theta [X_{-t}\big(X_{-t}^\top \overline \beta_t+\eta_t\big)] S_{t}^\top \Vert_2\\
&= \Vert \Phi^\top S_{-t}\Sigma_\theta \big(S_{-t}^\top\Phi {\fc} - \overline \beta_t\big) - \Phi^\top S_{-t} \E_\theta [X_t\eta_t] S_{t}^\top \Vert_2\\
&= \Vert \Phi^\top S_{-t}\Sigma_\theta \big(S_{-t}^\top\Phi {\fc} - \overline \beta_t\big) - \Phi^\top S_{-t} J(1-B_\theta)^{-1}e_t S_{t}^\top \Vert_2
\end{align*}
\end{proof}
The following Lemma is similar to Lemma \ref{lemma:general}; we have included it for completeness.

\begin{lemma}\label{lemma:indirect}
For some matrix $U_\theta$, we have $$\Phi^\top S_{-t}\Sigma_\theta \big(S_{-t}^\top\Phi {\fc} - \overline \beta_t\big) - \Phi^\top S_{-t} J(1-B_\theta)^{-1}e_t S_{t}^\top  = U_\theta r(\Phi)^{(2)}.$$
\end{lemma}
\begin{proof}
Observe that
$$\Phi^\top S_{-t}\Sigma_\theta S_{-t}\Phi\fc = U^1_\theta r(\Phi)^{(\dot 2)}$$
where $$(U^1_\theta)_{a, ((ij),(kl))} = 
\begin{cases}
\big(S_{-t}\Sigma_\theta S_{-t}^\top\big)_{jk}(\fc)_l, & a = i\text{ and }(ij)\ne \cdot\text{ and } (kl)\ne \cdot\\
0, & a\ne i \text{ or } (ij) = \cdot \text{ or } (kl)= \cdot
\end{cases}
$$
The second term is almost in the form we would like already
$$\Phi^\top (S_{-t}\Sigma_\theta\overline \beta_t-J(1-B_\theta)^{-1}e_tS_t^\top) = U^2_\theta r(\Phi)^{(\dot 2)}$$
where $$(U^2_\theta)_{a, ((ij),(kl))} = 
\begin{cases}
(S_{-t}\Sigma_\theta\overline\beta_t-S_{-t}J(1-B_\theta)^{-1}e_tS_t^\top)_{ij}&  (ij)\ne \cdot\text{ and } (kl)= \cdot\\
0, & (ij)= \cdot\text{ or } (kl)\ne \cdot\\
\end{cases}
$$
Our result follows by setting $U_\theta  = U_\theta ^1+U_\theta ^2$.
\end{proof}

Note that the indirect model contains as a special case the latent variable model, in which we only observe a subset of the variables.
\begin{theorem}\label{appendix_thm:indirect}

Suppose that we are given $O(\frac{n^{4}+\log\frac{1}{\delta}}{\delta'})$ interventions drawn independently from distribution $\D$ over the intervention index set $\Theta$. Then with probability at least $1-\delta$ over the randomness in $\ThetaTr$, the following statement holds: 
$$\Pr_{\theta \sim \D} (\I^\epsilon_{\fc}(\ThetaTr)\notin \I^\epsilon_{\fc}(\theta)) \leq \delta'. $$
\end{theorem}
\begin{proof}
We again begin with our expression for the gradient of the loss given by Lemma \ref{lemma:indirect_gradient_expansion}.
$$\nabla \E_\theta[f^\top \Phi Z_{-t} - Z_t] \mid_{\fc} = \Phi^\top S_{-t}\Sigma_\theta \big(S_{-t}^\top\Phi {\fc} - \overline \beta_t\big) - \Phi^\top S_{-t} J(1-B_\theta)^{-1}e_t S_{t}^\top.$$
We show in Lemma \ref{lemma:indirect} that we can write this as 
$\nabla \R^\theta_\Phi(\fc) = U_\theta r(\Phi)^{(\dot2)}$ for some matrix $U_\theta $. From Lemma \ref{lemma:normsquare} we can write the squared norm as 
$\Vert \nabla \R^\theta_\Phi(\fc)\Vert_2^2 = r(U_\theta ^\top U_\theta )(r(\Phi)^{(\dot 2)})^{(2)}.$ The result follows identically to Theorem \ref{thm:general}.
\end{proof}



\section{Finite sample bounds}
\begin{lemma}\label{lemma:singlePhi}
Given $\frac{dL^2}{\epsilon^2}\log\frac{2d}{\delta}$ samples from $\Delta_\theta$, we can compute $\Vert \hat \nabla \R^\theta_\Phi(\fc)\Vert$ such that with probability $1-\delta$ over the samples drawn \textit{in each interventional distribution}, for any fixed $\Phi$,
$$\Vert \nabla \R^\theta_\Phi(\fc)\Vert - \Vert\hat \nabla \R^\theta_\Phi(\fc)\Vert_2 \le \epsilon.$$
\end{lemma}
\begin{proof}
Suppose we had $2m$ samples from $\Delta_\theta$. Split these samples into $m$ pairs $(X^{2i}, X^{2i+1})_{i=1}^m$ randomly. Denote by $R^\theta_\Phi(f, X)$ the loss at a point $X$, given by $$R^\theta_\Phi(f, X) = \Phi^\top XX^\top \Phi \fc-\Phi^\top XX_t$$
Then we can estimate $\Vert \nabla \R^\theta_\Phi(\fc)\Vert$ as
$$\Vert \hat \nabla \R^\theta_\Phi(\fc)\Vert = \frac{1}{m}\sqrt{\left(\sum_{i=1}^m R^\theta_\Phi(\fc, X^{2i})\right)^\top\left(\sum_{i=1}^m R^\theta_\Phi(\fc, X^{2i+1})\right)}$$
Each of the inner terms concentrate about their mean as given by Lemma \ref{lemma:vectorHoefding}.
\begin{align*}
    &\frac{1}{m}\sqrt{\left(\sum_{i=1}^m R^\theta_\Phi(\fc, X^{2i})\right)^\top\left(\sum_{i=1}^m R^\theta_\Phi(\fc, X^{2i+1})\right)}-\Vert \R^\theta_\Phi(\fc)\Vert\\
    &\hspace{0.5cm}\le 
    \frac{1}{m}\sqrt{\left(\sum_{i=1}^m R^\theta_\Phi(\fc, X^{2i})-\R^\theta_\Phi(\fc)\right)^\top\left(\sum_{i=1}^m R^\theta_\Phi(\fc, X^{2i+1})-\R^\theta_\Phi(\fc)\right)}\le \epsilon
\end{align*}
where the final inequality is true for $m\ge\frac{dL^2}{\epsilon^2}\log\frac{2d}{\delta}$ with probability $1-\delta$.
\end{proof}
\begin{lemma}\label{lemma:coveringnumber}
Consider the function class consisting of functions parameterized by $\Phi$ for a fixed $\theta$
$$\{\R^\theta_\Phi(f):= \Phi^\top \Sigma_e\Phi f - \Phi^\top \E_\theta[XX_t]\hspace{0.5cm}\big|\hspace{0.5cm}\Vert \Phi\Vert_2 = 1, \Vert f\Vert = \sqrt{n}\}$$
Let $\N_\epsilon$ denote the covering number of this set with respect to the metric $\text{dist}(\Phi_1, \Phi_2) = \Vert \R^\theta_{\Phi_1}(f)-\R^\theta_{\Phi_2}(f)\Vert$. Then $\log \N_\epsilon =n^2\log \left(1+\frac{4n^{3/2}}{\epsilon}\right)$.
\end{lemma}
\begin{proof}
We have
\begin{align*}
    \text{dist}(\Phi_1, \Phi_2) 
    &= \Vert \R^\theta_{\Phi_1}(f)-\R^\theta_{\Phi_2}(f)\Vert\\
    &= \Vert \Phi_1^\top \Sigma_\theta\Phi_1 f -\Phi_2^\top \Sigma_\theta\Phi_2 f + \Phi_2^\top \E_\theta[X_{-t}X_t]- \Phi_1^\top \E_\theta[X_{-t}X_t]\Vert\\
    &\le \Vert \Phi_1^\top \Sigma_\theta\Phi_1 f -\Phi_2^\top \Sigma_\theta\Phi_2 f \Vert+ \Vert\Phi_2^\top \E_\theta[_{-t}XX_t]- \Phi_1^\top \E_\theta[X_{-t}X_t]\Vert\\
    &\le \Vert \Phi_1^\top \Sigma_\theta\Phi_1 -\Phi_2^\top \Sigma_\theta\Phi_2\Vert\Vert f\Vert+ \Vert\Phi_1-\Phi_2\Vert \Vert \E_\theta[X_{-t}X_t]\Vert\\
    &\le \Vert (\Phi_1-\Phi_2)^\top\Sigma_\theta(\Phi_1 -\Phi_2)\Vert\Vert f\Vert+ \Vert\Phi_1-\Phi_2\Vert \Vert \E_\theta[X_{-t}X_t]\Vert\\
    &\le \Vert \Phi_1-\Phi_2\Vert^2\Vert\Sigma_\theta\Vert\Vert f\Vert+ \Vert\Phi_1-\Phi_2\Vert \Vert \E_\theta[X_{-t}X_t]\Vert
\end{align*}
We have $\Vert \fc\Vert = \sqrt{n}$. Because $\Vert X\Vert_\infty \le 1$, we know that $\Vert \Sigma_\theta\Vert\le n$ and $\Vert \E_\theta[X_{-t}X_t]\Vert \le \sqrt{n}$. Now we take $\phi$ to be an $\frac{\epsilon}{2n^{3/2}}$-covering of the space of representations $\Phi$. We know from Lemma \ref{lemma:coveringnumbersphere} that we can find one containing at most $\left(1+\frac{4n^{3/2}}{\epsilon}\right)^{n^2}$ representations. That is, given any $\Phi, \Vert \Phi\Vert \le 1$, we can find $\Phi'\in \phi$ such that $\Vert \Phi-\Phi'\Vert \le \frac{\epsilon}{2n^{3/2}}$. Then we have (assuming $\epsilon<1, n> 1$)
$$\text{dist}(\Phi, \Phi') \le \left(\frac{\epsilon}{2n^{3/2}}\right)^2n^{\frac{3}{2}}+\frac{\epsilon}{2n^{3/2}} \sqrt{n}\le \epsilon.$$

\end{proof}
Putting these together,
\begin{lemma}\label{lemma:overallfinitesample}
Given $\frac{4dL^2}{\epsilon^2}\left(\log\frac{2d}{\delta}+\log\N_{\epsilon/2}\right)$ samples in a dataset, with probability $1-\delta$, for every representation $\Phi$ simultaneously we have 
$$\hat \I^{\frac{\epsilon}{2}}_{\fc}(\theta) \subseteq \I^\epsilon_{\fc}(\theta)$$
\end{lemma}
\begin{proof}
Consider $\Phi\in \hat \I^{\frac{\epsilon}{2}}_{\fc}(\theta)$. From the definition, $$\Vert \hat \nabla\R^\theta_\Phi(\fc)\Vert < \frac{\epsilon}{2}.$$ By Lemma \ref{lemma:coveringnumbersphere} and Lemma \ref{lemma:singlePhi}, we know that
$$\Vert \nabla R^\theta_\Phi(\fc)\Vert - \Vert \hat \nabla\R^\theta_\Phi(\fc)\Vert < \frac{\epsilon}{2}$$
Putting these together, we see that 
$$\Vert \nabla R^\theta_\Phi(\fc)\Vert \le \epsilon\implies \Phi\in \I^\epsilon_{\fc}(\theta).$$
\end{proof}

\begin{theorem}\label{appendix_thm:main_thm}
Given 
$$m = 
\begin{cases}O(\frac{k^4+\log\frac{1}{\delta}}{\delta'}) & \text{$\Theta$ is $k$ nodes, hard interventions}\\
O(\frac{d^{4k}+\log\frac{1}{\delta}}{\delta'}) & \text{$\Theta$ is $k$ nodes, soft interventions}\\
O(\frac{n^{4}+\log\frac{1}{\delta}}{\delta'}) & \text{$\Theta$ any interventions}\\
\end{cases}$$
interventional datasets, and $\frac{4nL^2}{\epsilon^2}\left(\log\frac{2nm}{\delta}+n^2\log(1+\frac{8n^{3/2}}{\epsilon})\right)$ samples in each dataset, we have that with probability $1-\delta$, with probability $1-\delta'$ for $\theta\sim\D$,
$$\hat\I^\epsilon_{\fc}(\ThetaTr) \subseteq \I^{2\epsilon}_{\fc}(\theta).$$
\end{theorem}
\begin{proof}
From Lemma \ref{lemma:singlePhiMain} we know that \begin{equation}\label{eq:singletheta}
\hat \I^\epsilon_{\fc}(\theta) \subseteq \I^{2\epsilon}_{\fc}(\theta).
\end{equation}
With a union bound over $\ThetaTr$, we can establish Equation (\ref{eq:singletheta}) at once over all interventions in the training set. Since $\I^\epsilon_{\fc}(\ThetaTr) = \bigcap_{\theta\in\ThetaTr} \I^\epsilon_{\fc}(\theta)$ and $\hat \I^\epsilon_{\fc}(\ThetaTr) = \bigcap_{\theta\in\ThetaTr} \hat \I^\epsilon_{\fc}(\theta)$, we get that $\hat \I^\epsilon_{\fc}(\ThetaTr)\in \I^\epsilon_{\fc}(\ThetaTr)$. Finally, from Theorems  \ref{thm:general}, \ref{thm:hard}, \ref{thm:soft}, we have $\I^\epsilon_{\fc}(\ThetaTr) \subseteq \I^\epsilon_{\fc}(\Theta)$ for each of the cases listed.
\end{proof}

\section{Empirical study}
\label{appendix:empirical}

\begin{figure}
  \centering
  \includegraphics[width = 1.1\linewidth]{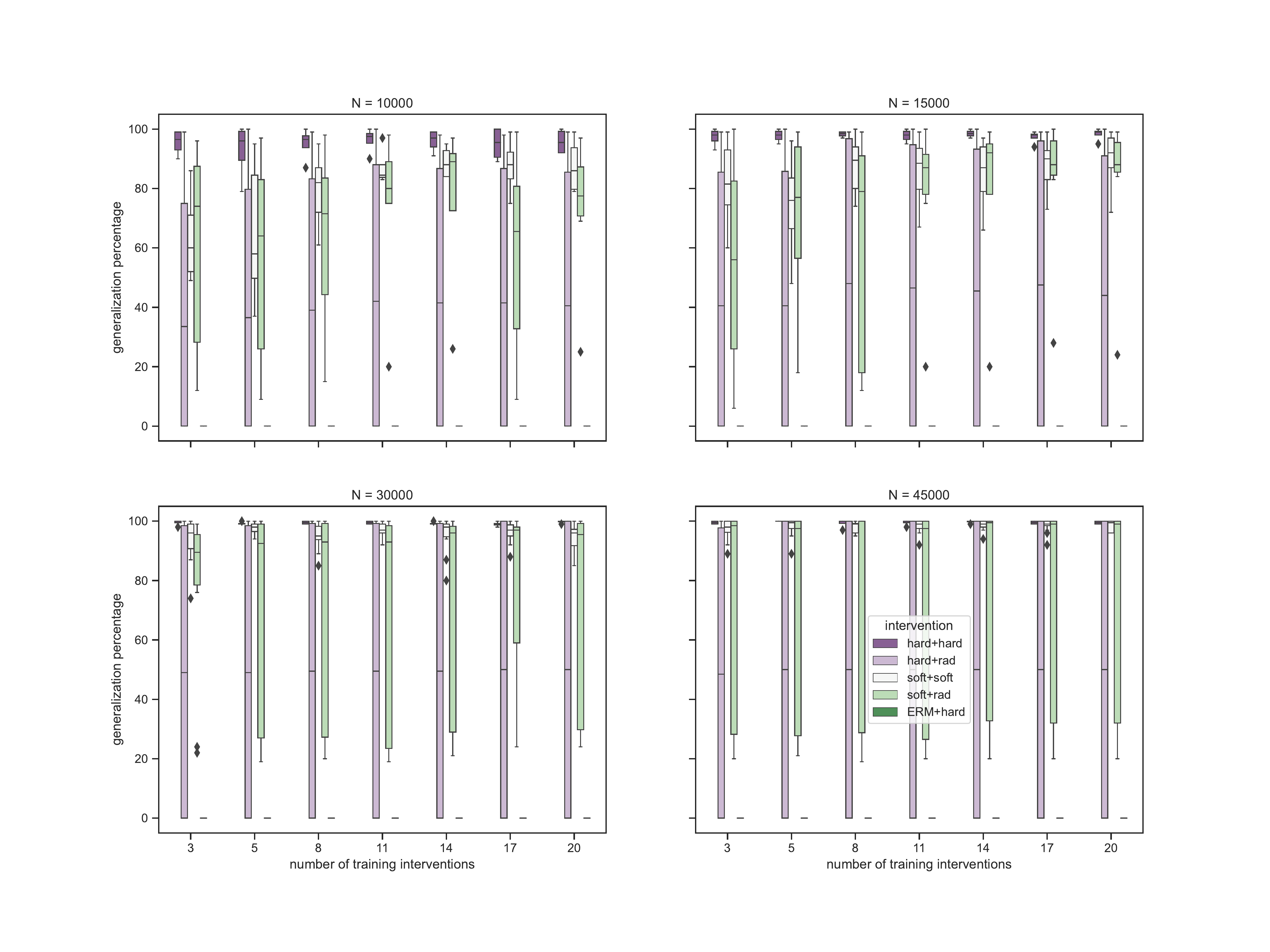}
  \caption{Comparison of generalization for various combinations of training and test interventional distributions. Here ``hard" refers to the i.i.d. interventions setting with hard interventions from the same distribution in both training and test interventions, similarly ``soft" refers to i.i.d. soft interventions, while  ``hard+rad" (respectively ``soft+rad") refer to hard interventions (respectively soft) in training with rademacher interventions in test. \textit{Right: } We repeat the above for $N = 15,000, 30,000$, and $45,000$.}\label{fig:extended}
\end{figure}

\textbf{Construction of training and test interventional distributions:}
For the class of hard interventions, we consider assignments drawn from Gaussians on nodes $\{v_1, v_2, v_3\}$. Specifically, we take $X_i = a_i$ for $i\in \{3,4,5\}$ with $a_i\sim \N(0, 1)$ independently for each node. For the class of soft interventions, we assign each nonzero entry of $\beta^\theta_i \sim \N(0, 1)$ for $i\in \{3,4,5\}$ independently. For Figure \ref{fig:extended}, we use a noise variance for $\eta_i$ of $0.02$ and take $N = 30,000$ generate samples per dataset. In this way, we generate $m$ atomic-interventional datasets $\{N_{j, \text{hard}}\}_{j\in [m]}$ and $m$ soft-interventional datasets $\{N_{j, \text{soft}}\}_{j\in [m]}$ for $m$ varying from $3$ to $20$.

\textbf{Subsets as Representations: }We iterate over the power-set of the nodes to make a list of approximately invariant representations. For every subset $S$ of the non-target nodes, we consider $\Phi_S$ to be the corresponding representation, diagonal, with a one on each index in $S$ (a projection onto $S$). We count the fraction of these representations that are approximately invariant, as defined below.

\textbf{Approximately Invariant Representations: }Here, these are defined as being such that there is not much variation over the least squares solutions on top of the representations. For every such representation $\Phi_S$, we denote by $f_{S, j}$ the least squares solution on top of the representation $\Phi_S$ for dataset $N_j$. We take as a measure of invariance the quantity $$\rho_S = \frac{\sum_{j_1, j_2\in [m]}\Vert f_{S, j_1}-f_{S, j_2}\Vert_2}{(m-1) \Vert \sum_{j\in [m]} f_{S, j}\Vert_2}.$$
Note that this is the ratio of the average distance between the various least squares solutions and the norm of the average least squares solution. We expect this to be low when the various least squares solutions are close to one another. An approximately invariant subset $S$ is then taken to be one such that $\rho_S<0.02$. Let $s_{\text{hard}}$,  $s_{\text{soft}}$ denote the approximately invariant subsets given the training sets $\{N_{j, \text{hard}}\}_{j\in [m]}$, $\{N_{j, \text{soft}}\}_{j\in [m]}$

\begin{wrapfigure}{r}{0.3\textwidth}
\centering
\includegraphics[]{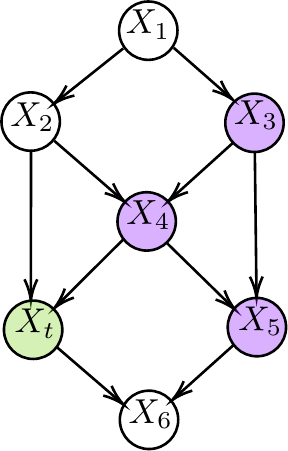}
\caption{Linear SEM for the empirical demonstrations.}
\end{wrapfigure}

\textbf{Test distributions:} We construct $m_{\text{test}}$ test datasets induced by interventions drawn from each of the same families, that is, $\{N^{\text{test}}_{j, \text{hard}}\}_{j\in m_{\text{test}}}$ for atomic interventions, and $\{N^{\text{test}}_{j, \text{hard}}\}_{j\in m_{\text{test}}}$ for soft interventions, as well as a set of datasets $\{N^{\text{test}}_{\alpha, j, \text{rad}}\}_{j\in m_{\text{test}}}$, drawn from the linear SEM constructed by flipping each edge weight in the original SEM with probability $\alpha$ (except the ones connecting the parents of the target to the target). Note that this last dataset is not drawn from interventions that are drawn from the same distribution as the training datasets. These sets are of the same size $N = 30,000$ as the training datasets, and use the same noise variance for $\eta_i$, $0.02$. For each set in $s_{\text{hard}}$, we return the percentage of datasets in $\{N^{\text{test}}_{j, \text{hard}}\}_{j\in m_{\text{test}}}$, $\{N^{\text{test}}_{j, \text{soft}}\}_{j\in m_{\text{test}}}$, and $\{N^{\text{test}}_{\alpha, j, \text{rad}}\}_{j\in m_{\text{test}}}$ such that the least squares solution on top of the subsets continues to be approximately the same. The labels of the plots reflect in order the trianing and test interventional ditributions, so ``soft+rad($0.5$)" means we are evaluating $s_{\text{soft}}$ on $100$ datasets $\{N_{0.5, j, \text{rad}}\}$.

\textbf{ERM generalization:} To demonstrate that generalization in this sense is indeed non-trivial, we also plot the fraction of least squares solutions to the observational datasets (where we pool together the training distributions) that exhibit low variance in test interventional distributions. This box plot is essentially trivial at $0$.
